\documentclass[12pt,onecolumn,draftcls]{IEEEtran}
\def\comment#1{}
\usepackage{graphicx,graphics,epstopdf,subfig}
\usepackage{amsmath,amssymb}
\usepackage{algorithm,algorithmic}
\usepackage{enumerate}
\usepackage{xspace}
\usepackage{color}
\usepackage{multirow}
\usepackage{hyperref}
\usepackage[noadjust]{cite}
\newcommand{\RN}[1]{%
	\textup{\uppercase\expandafter{\romannumeral#1}}%
}

\newcommand{\Amat}{{\bf A}}

\newcommand{\Imat}{{\bf I}}

\newcommand{\Rmat}[0]{{{\bf R}}}
\newcommand{\Smat}[0]{{{\bf S}}}
\newcommand{\Tmat}[0]{{{\bf T}}}

\newcommand{\Wmat}[0]{{{\bf W}}}

\newcommand{\nv}{\boldsymbol{n}}

\newcommand{\xv}{\boldsymbol{x}}
\newcommand{\yv}{\boldsymbol{y}}

\newcommand{\zv}{\boldsymbol{z}}

\newcommand{\alphav}{\boldsymbol{\alpha}}
\newcommand{\betav}{{\boldsymbol{\beta}} }

\newcommand{\zetav}{{\boldsymbol{\zeta}} }

\newcommand{\thetav}{\boldsymbol{\theta}}

\newcommand{\lambdav}[0]{{\boldsymbol{\lambda}} }

\newcommand{\ts}{^{\top}}

\newtheorem{definition}{Definition}
\newtheorem{lemma}{Lemma}
\newtheorem{theorem}{Theorem}
\newtheorem{proof}{Proof}
\newtheorem{corollary}{Corollary}
\begin{document}
	\setlength{\parskip}{.02in}
\title{Lensless Compressive Imaging}

\author{
	\authorblockN{Xin Yuan, Hong Jiang, Gang Huang and Paul Wilford}
		\authorblockA{Bell Labs, Alcatel-Lucent, 600 Montain Avenue, Murray Hill, NJ, 07974, USA.\\
		 \{x.yuan, hong.jiang, gang.huang, paul.wilford\}@alcatel-lucent.com}
	}



\maketitle

\begin{abstract}
We develop a lensless compressive imaging architecture, which consists of an aperture assembly and a single sensor, without using any lens. 
An {\em anytime} algorithm is proposed to reconstruct images from the compressive measurements; the algorithm produces a sequence
of solutions that monotonically converge to
the true signal (thus, anytime).
The algorithm is developed based on the sparsity of local overlapping patches (in the transformation domain) and state-of-the-art results have been obtained.
Experiments on real data demonstrate that encouraging results are obtained by measuring about 10\% (of the image pixels) compressive measurements. 
The reconstruction results of the proposed algorithm  are compared with the JPEG compression (based on file sizes) and the reconstructed image quality is close to the JPEG compression, in particular at a high compression rate.
\end{abstract}



\maketitle

\section{Introduction}
Compressive sensing~\cite{Candes05compressed,donoho2006compressed} is an emerging technique to acquire and process digital data such as two-dimensional images~\cite{Duarte08SPC,Romberg08SPM,Yuan14TSP}, hyperspectal images~\cite{Wagadarikar09CASSI,Yuan15JSTSP,Tsai15OL}, polarization images~\cite{Tsai15OE} and videos~\cite{Patrick13OE,Yuan14CVPR,Yang14GMM,Yang14GMMonline,Yuan13ICIP,Li13TB,Jiang12Bell,Jiang12Inverse,HollowayICCP12,Veeraraghavan11TPAMI,Hitomi11ICCV,Reddy11CVPR}. Compressive sensing is most effective when it is used in data acquisition: to capture the data in the form of compressive measurements~\cite{Goyal08SPM}. 
Though the theory has been established over a decade ago~\cite{Candes05compressed,Candes05,donoho2006compressed}, practical applications are still attracting researchers and engineers. 
With compressive measurements, images may be reconstructed with far fewer measurements than the number of pixels in the original images. Therefore, by using compressive sensing in acquisition, images are compressed while they are captured, avoiding high speed processing, or transmission, of a large number of pixels.
The single pixel camera~\cite{Takhar06SPIE,Duarte08SPM} directly captures compressive measurements of an image, which is a camera architecture that employs a digital micromirror array to perform optical implementation of linear projections of an image onto pseudo-random binary patterns. It has the ability to obtain an image with a single detection element while sampling the image fewer times than the number of pixels. The same camera architecture is also used for Terahertz imaging~\cite{Chan08APL,Heidari09MMW}, and millimeter wave imaging~\cite{{Babacan11ICIP}} and X-ray imaging~\cite{Wang15SIAM_xray}. These cameras all make use of a lens to form an image in a plane before the image is projected onto a pseudo-random binary pattern. Lenses, however, severely constrain the geometric and radiometric mapping from the scene to the image~\cite{Zomet06CVPR}. Furthermore, lenses add size, cost and complexity to a camera, especially at the non-visible bandwidth.

In this paper, we present an in-depth description and mathematical analysis of a lensless compressive imaging architecture that was originally proposed in~\cite{Huang13ICIP}. We have built a new version of lensless camera with new hardware under the same architecture, which consists of two components, an aperture assembly and a single sensor. No lens is used. The aperture assembly consists of a two dimensional array of aperture elements. The transmittance of each aperture element is independently controllable. The sensor is a single detection element, such as a single photo-conductive cell. Each aperture element together with the sensor defines a cone of a bundle of rays, and the cones of the aperture assembly define the pixels of an image. The sensor is used for taking compressive measurements. Each measurement is the integration of rays in the cones modulated by the transmittance of the aperture elements.
The proposed architecture is different from the cameras of~\cite{Takhar06SPIE} and ~\cite{Zomet06CVPR}. The fundamental difference is how the image is formed. In both~\cite{Takhar06SPIE} and~\cite{Zomet06CVPR}, an image of the scene is formed on a plane, by some physical mechanism such a lens or a pinhole, before it is digitally captured (by compressive measurements in~\cite{Takhar06SPIE}, and by pixels in ~\cite{Zomet06CVPR}). In the proposed architecture of this work, no image is physically formed before the image is captured. 
The proposed architecture is also related to, the traditional coded aperture imaging~\cite{Caroli87SCR,Zand13CodeAperture}. Specifically, if the sensor number is on the order of the aperture element figures, the proposed architecture will be the coded aperture.
However, only a single detector is used in our system. 
The proposed architecture is distinctive with the following features.
\begin{itemize}
	\item 
	No lens is used. An imaging device using the proposed architecture can be built with reduced size, weight, cost and complexity. In fact, our architecture does not rely on any physical mechanism to form an image before it is digitally captured.
	\item 
	No scene is out of focus. The sharpness and resolution of images from the proposed architecture are basically limited by the resolution of the aperture assembly (number of aperture elements), there is no blurring introduced by lens for scenes that are out of focus.\footnote{When the scene is far from the camera, each pixel covers a large area and the image will look blurry. However, this is not introduced by the lens as in a conventional camera.}
	\item 
	The same architecture can be used for imaging of visible spectrum, and other spectra such as infrared and millimeter waves.
	\item 
	When multiple sensors are used in our system, they can be placed in arbitrary position and the scene is still in focus for each sensor. Therefore, it readily forms a multi-view system~\cite{Jiang14APSIPA}. 
	\item The proposed architecture can be used to capture hyperspectral images~\cite{Yuan15JSTSP} and polarized images~\cite{Tsai15OE} by integrating related hardware. 
\end{itemize}
We built a prototype device for capturing images of visible spectrum. It consists of an LCD panel, and a photo-electric detector. 
Though the camera has been introduced in~\cite{Huang13ICIP,Jiang14APSIPA}, no algorithm has been developed particularly for this new camera and in this paper we present more details on both  hardware and algorithm development.

The proposed algorithm is an {\em anytime} algorithm~\cite{Liao14GAP}.
As defined in~\cite{Zaimag96}, ``anytime algorithms are algorithms whose quality of results improves gradually as computation time increases".
Specifically, the solution in each iteration of our algorithm {\em monotonically} converges to the ground truth. 
In addition, we compare our new algorithm with JPEG compression at (roughly) the same compression ratio (based on the size of JPEG files). More importantly, we demonstrate that the proposed algorithm achieves similar quality of JPEG at high compression ratio.

The rest of this paper is organized as follows:
The architecture is described in Section~\ref{Sec:Arc} and Section~\ref{Sec:Form} derives the mathematical formulation.
Our new algorithm is proposed in Section~\ref{Sec:Algo} and Section~\ref{Sec:sim} provides extensive simulation results.
Section~\ref{Sec:Hardware} presents the experimental hardware and real data results.
Section~\ref{Sec:Col} summarizes the entire paper.

\begin{figure}[hbtp]
	\centering
	\includegraphics[width=0.8\textwidth]{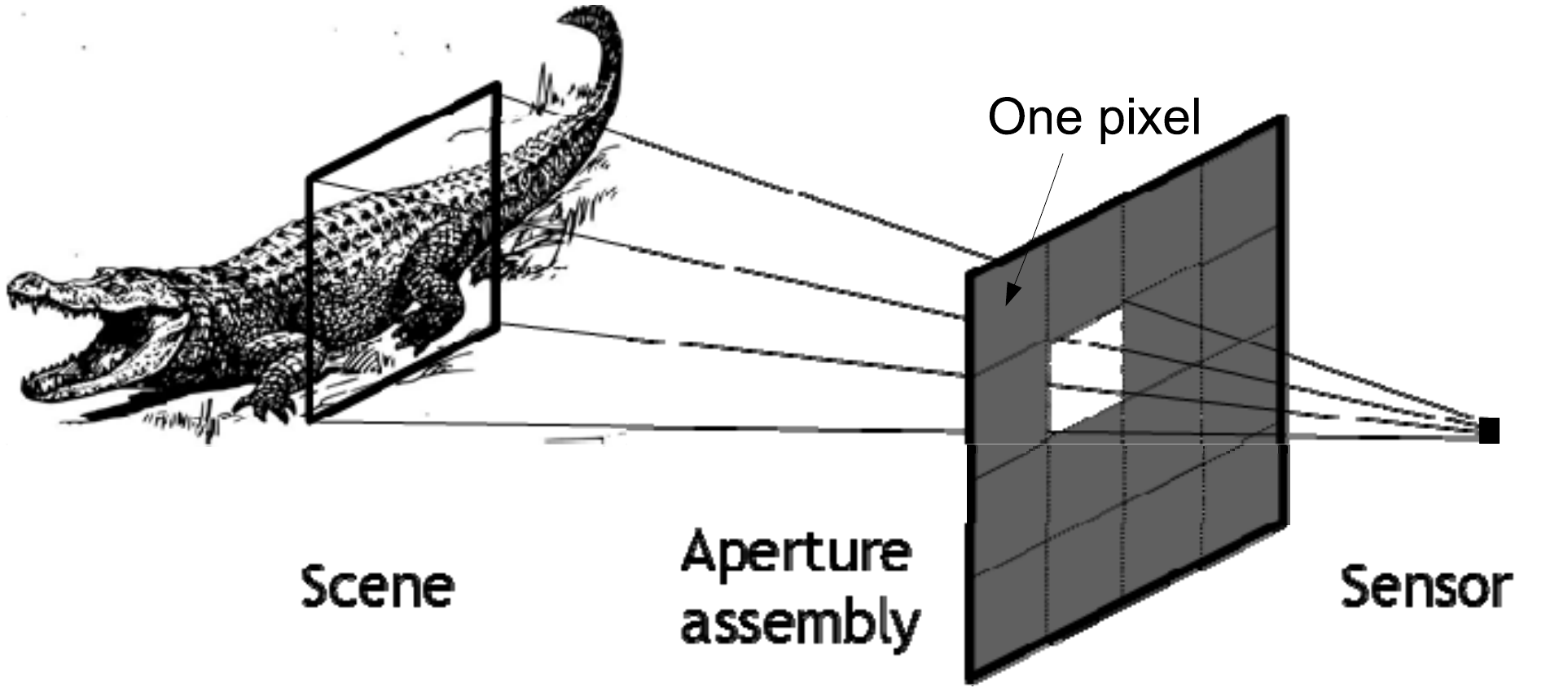}
	\caption{The proposed architecture. It consists of two components: an aperture assembly and an infinitesimal sensor of a single detection element.}
	\label{fig:arch}
\end{figure}

\section{Architecture}
\label{Sec:Arc}
Figure~\ref{fig:arch} depicts the proposed architecture, which consists of two principle components: an aperture assembly and a single sensor (a photodiode), which measures the light intensity to capture grayscale images, or a tri-color sensor which measures intensity of each RGB channel to capture color images. 
The aperture assembly is made up of a two dimensional array of aperture elements (blocks on the aperture assembly in Figure~\ref{fig:arch}); the transmittance of each element $\{P_{i,j}\}_{i,j =1}^{N_x, N_y}$ can be individually controlled with $(N_x, N_y)$ denoting the dimension of the aperture assembly. 

Each element of the aperture assembly, together with the sensor, defines a cone of a bundle of rays (Figure~\ref{fig:arch}), and the integration of the rays within a cone is defined as a pixel value of the image. Therefore, in the proposed architecture, an image is defined by the pixels which correspond to the array of aperture elements in the aperture assembly.
One possible way to measure an image with a single sensor is to capture the image pixel by pixel, which can be implemented by reading of the sensor when one of the aperture elements is completely open ($P_{i,j} = 1$) and all others are completely closed ($\{P_{i^{\prime},j^{\prime}}\}_{\forall i^{\prime}\neq i, j^{\prime} \neq j} = 0$). 
The measurements are the pixel values of the image when the elements of the aperture assembly are opened one by one in certain scan order. This corresponds to the traditional representation of a digital image pixel by pixel.

With compressive sensing, it is possible to represent an image by using fewer measurements than the number of pixels~\cite{Takhar06SPIE,Duarte08SPM}. The proposed architecture in Figure~\ref{fig:arch} aims to simplify the procedure of capturing compressive measurements.
%
%
%
Recall in compressive sensing~\cite{donoho2006compressed}
\begin{equation}\label{eq:yAx}
\yv = \Amat \xv + \nv, 
\end{equation}
where $\Amat\in {\mathbb R}^{M\times N}$ is the sensing matrix, $\xv$ is the desired signal (denoting vectorized image in this work and $N = N_x N_y$), $\yv \in{\mathbb R}^M$ is the measurement and usually $M\ll N$; ${\bf n}$ symbolizes the measurement noise.
Each row of $\Amat$ defines a pattern for the elements of the aperture assembly, and the number of columns in a sensing matrix is equal to the number of total elements in the aperture assembly. 
Each value in a row of the sensing matrix is used to define the transmittance of an element of the aperture assembly. A row of the sensing matrix therefore completely defines a pattern for the aperture assembly, and it allows the sensor to make one measurement (one element in $\yv$) for the given pattern of the aperture assembly. 
The number of rows of the sensing matrix is the number of measurements, which is usually much smaller than the number of aperture elements in the aperture assembly (the number of pixels). 
Let $\Amat$ be a matrix whose entries are random numbers between 0 and 1. To make a measurement, the transmittance, $P_{i,j}$, of each aperture element is controlled to equal the value of the corresponding entry in a row of the sensing matrix. The sensor integrates all rays transmitted through the aperture assembly. The intensity of the rays is modulated by the transmittances before they are integrated. Therefore, each measurement from the sensor is the integration of the intensity of rays through the aperture assembly multiplied by the transmittance of respective aperture element. 
A measurement from the sensor is hence a projection of the image onto the row of the sensing matrix. 
By changing the pattern of the transmittance of the aperture assembly,
the compressive measurement is captured corresponding to a given sensing matrix.

\section{Mathematical Formulation}
\label{Sec:Form}
In this section, we formally define what an image is in the proposed architecture and how it is related to the measurements from the sensor.
In particular, we will describe how a pixelized image can be reconstructed from the measurements taken from the sensor.

\subsection{Virtual Image on the Aperture Assembly}
The analog scene ${I}(u,v)$ can be defined on any plane between the scene and the sensor and for convenience, we here define the image on the aperture assembly.
Considering one point on the aperture assembly, there is a ray starting from a point on the scene, passing through the point $(u,v)$ on the aperture assembly, and ending at the sensor. Let $r(u,v; t)$ denote the intensity of this ray arriving at the sensor, passing through the aperture assembly $(u,v)$ at time $t$. The image point ${I}(u,v)$ can be defined by the integration of the ray in a time interval $\Delta t$
\begin{equation}\label{eq:Ixy}
{I}(u,v) = \int_0^{\Delta t} r(u,v; t) dt.
\end{equation}
It is worth noting that $I(u,v)$ is continuously defined in the region of the aperture assembly and can be considered as an analog image.

Similarly, define the continuous transmittance pattern of the aperture assembly as $P(u,v)$, The measurement collected by the sensor is the integration of the rays through the aperture assembly, modulated by the transmittance pattern $P(u,v)$,
\begin{equation}\label{eq:z}
z = \iint I(u,v) P(u,v) du dv.
\end{equation}
Equation (\ref{eq:z}) defines the measurement of the sensor based on the continuous image in (\ref{eq:Ixy}). In the following, we analyze how the pixel is defined and then we can get the discretized image.

\subsection{Pixelized Image}
Since only a single sensor is used in our system, the 
virtual image defined in (\ref{eq:Ixy}) can be pixelized by the aperture assembly, which is similar to the single-pixel camera~\cite{Duarte08SPM}. 
Considering each element of the aperture assembly of size $\Delta_u \times \Delta_v$, each pixel value of image can be represented by the integration of all the rays passing through the aperture element $(i,j)$: 
\begin{eqnarray}
I(i,j) &=& \int_{(i-1)\Delta_u}^{i\Delta_u} \int_{(j-1)\Delta_v}^{j \Delta_v} I(u,v) du dv .
\end{eqnarray}
Image $\{I(i,j)\}_{i,j=1}^{N_x, N_y}$ can be vectorized to a long vector ${\bf I} \in {\mathbb R}^{N}$, which is the target signal $\xv$ in the compressive sensing model (\ref{eq:yAx}) and the image size is of $N_x \times N_y$.

\subsection{Compressive Measurement} 
When the aperture assembly is programmed to implement a compressive sensing matrix, the transmittance $P(i,j)$  of each aperture element is controlled to equal the value of the corresponding entry in the sensing matrix. 
For the $m$-th measurement, the entries in row $m$  of the sensing matrix are used to program the transmittance of the aperture elements. 
Specifically, let the sensing matrix ${\bf A}$  be a matrix whose entries, $a_{i,j}$, are random numbers between 0 and 1. Let $P^{m}(i,j)$ be the transmittance of aperture element $(i,j)$ for the $m$-th measurement. 
Following (\ref{eq:z}), the $m$-th measurement can be represented as
\begin{eqnarray}\label{eq:zm}
z_m &=& \iint P^{m}(u,v) I(u,v) du dv= \sum_{i=1}^{N_x} \sum_{j=1}^{N_y} I(i,j) P^m(i,j),
\end{eqnarray}
where we consider that each aperture element, inside the region $[(i-1)\Delta_u,i\Delta_u] \times [(j-1)\Delta_v,j \Delta_v]$, $P^m(u,v) = P^m(i,j)$, is a constant programmed by the user.
Let $\xv$ denote the vectorized formulation of $I$, and then we have $z_m = \sum_n a_{m,n} x_n$.
After taking $M$ measurements, in the noiseless case, we can write the sensing process as 
\begin{equation}
{\bf \zv} = {\bf A \xv},
\end{equation}
which is now the compressive sensing formulation as in (\ref{eq:yAx}) (if the noise is considered).
Then, our problem becomes that given ${\bf A}$ (designed and {\em known  a priori}) and the measurements $\bf \zv$, how to reconstruct the image $\bf \xv$.
A new algorithm, which explores the sparsity of the local region in a transformation domain, {\em e.g.,} the DCT (Discrete Cosine Transformation) used in the JPEG compression, is proposed below to achieve the state-of-the-art reconstruction.

\section{Reconstruction Algorithm}
\label{Sec:Algo}
The theory developed for compressive sensing~\cite{Candes05compressed} requires that the target signal $\xv$ is sparse and following this, researchers have extended the sparsity of the signal in a transformation domain. For example, the wavelet transformation~\cite{Yuan14TSP} is generally used in the image reconstruction of compressive sensing.
Since the wavelet is usually imposed  globally on the entire image, we term this as global transformation (the Total Variation used in~\cite{Li13COA} is also performed globally). Under these transformations, the coefficients usually have the same (or similar) number of the image pixels and these coefficients are approximately sparse (or compressible~\cite{Yuan14TSP}). 
A variety of algorithms (see references in~\cite{Mertzler14Denoising}) have been developed to explore the sparsity of these coefficients.
Let $\Tmat$ denote the basis of transformation, the signal $\xv$ can be represented as
\begin{equation}\label{eq:xTtheta}
\xv= {\bf T \thetav},
\end{equation} 
where $\thetav$ denotes the coefficients in the transformation domain.
Plug (\ref{eq:xTtheta}) into (\ref{eq:yAx}), we have
\begin{equation}\label{eq:yxTtheta}
\yv = {\bf A T \thetav} = {\bf R \thetav}.
\end{equation} 
Most algorithms have been developed (for example~\cite{Figueiredo07GPSR,Daubechie04IST}) to solve:
\begin{equation}\label{eq:theta_l1}
\min_{\thetav} \|{\bf \thetav}\|_1, \quad {\text{subject to}} \quad {\yv = \Rmat \thetav},
\end{equation} 
or the variates of this problem with $\|\cdot\|_1$ denoting the $\ell_1$ norm.
Advantages of these algorithms include fast computation and low memory cost by assuming that ${\bf T}$ is easily invertible.
When ${\bf T}$ is an orthonormal matrix, such as the wavelet transformation, solving $\bf \thetav$ in (\ref{eq:theta_l1}) is equivalent to solve $\bf \xv$. However, in other cases as stated below, there may be no one by one correspondence between $\bf \xv$ and $\bf \thetav$.
Recently, researchers have found that by exploiting the local (region) sparsity of the image can achieve better results than the global transformation methods, examples including the low-rank regularizer~\cite{Dong14TIP} and the denoising based method~\cite{Mertzler14Denoising}.

\subsection{Exploring the Local Sparsity}
Inspired by the JPEG compression and the emerging dictionary learning algorithms~\cite{Aharon06TSP} for local patches, researchers have developed algorithms based on the sparsity of the local patches in specific basis or learned dictionaries~\cite{Zhang14SP}.
Consider a general case, in which $\bf T$ in (\ref{eq:yxTtheta}) is now not a linear independent basis, but a more general dictionary, ${\bf D} \in {\mathbb R}^{q\times p}$ and usually $p\gg q$.
Let $\tilde {\bf X} \in{q \times N_p}$ denote the patch formulation of the image $\bf \xv$, with $n$ denoting the vectorized patch length ({\em e.g.}, a $\sqrt{q} \times \sqrt{q}$ two dimensional patch) and $N_p$ symbolizing the number of patches extracted from the image $\xv$.
We can write $\tilde{\bf X} = {\bf Q} \xv$, where ${\bf Q}$ denotes an extraction and permutation matrix.  
Under the dictionary ${\bf D}$,
\begin{eqnarray}\label{eq:xDalpha}
\tilde {\bf X} = {\bf D \Smat},
\end{eqnarray}
where ${\Smat} \in {\mathbb R}^{p \times N_p}$ is a matrix whose columns are the coefficients of each patch and it is usually sparse.
Recall the key of compressive sensing is to find a sparse basis and with the formulation in (\ref{eq:xDalpha}), $\bf D$ is the basis and ${\Smat}$ plays the role of sparsity.

\subsection{Formulation of the Reconstruction}
By adapting the above formulation, the compressive sensing problem is no longer the same as in (\ref{eq:theta_l1}), but can be formulated as an iterative two-step procedure:
\begin{itemize}
	\item Step 1: To minimize the following objection function
	\begin{equation} \label{eq:Jx}
	J({\bf \xv}) = \|{\bf \yv - A\xv}\|_2^2.
	\end{equation}
	\item Step 2: To solve the following minimization problem
	\begin{equation} \label{eq:alpha_L1}
	\min_{\Smat} \|{\Smat}\|_1, \quad {\text{subject to}} \quad {\tilde{\bf X} = {\bf D \Smat}}.
	\end{equation}
\end{itemize}  
It is worth noting that in step 1, we don't need ${\bf \xv}$ to be sparse, but when ${\bf A}$ is a compressive sensing matrix, a regulizer term is needed, {\em e.g.}, the TV used in~\cite{Huang14TIP}. An alternative solution is to use the majorization-minimization (MM) approach~\cite{Figueiredo07MM}, which will be described below.
In step 2, when ${\bf D}$ is given, (\ref{eq:alpha_L1}) is the conventional sparse coding problem (or, for each column of $\Smat$, it is compressive sensing problem).

\subsection{Proposed Algorithm: SLOPE}
Different from the formulation in (\ref{eq:theta_l1}), for which diverse algorithms have been developed to solve the unique $\bf \thetav$, thus to obtain ${\bf \xv}$, in the above formulation, we aim to get ${\bf \xv}$ directly and this is also the final target of reconstruction.
Therefore, step 2 can be recognized as a denoising step, while step 1 aims to update ${\bf \xv}$.
In this paper, we solve step 1 with the Euclidean projection~\cite{Liao14GAP}, which, under the condition of the sensing matrix (Hadamard matrix) used in our camera, is same as the iterative shrinkage/thresholding (IST) derived from the MM. However, we prove that a larger range of the step-size (than the IST) still leads to good convergence.  

\subsubsection{Update ${\bf \xv}_k$}
\label{Sec:update_x}
Under the compressive sensing framework, (\ref{eq:Jx}) has a solution in closed form, which is to use pseudo-inversion 
${\bf \xv}={\bf A}^{\top} ({\bf A}{\bf A}^{\top})^{-1} {\bf \yv}$. By using the MM approach to minimize $J({\bf \xv})$, we can avoid solving a system of linear equations. At each iteration $k$ of the MM approach, we should find a function $G_k({\bf \xv})$ that coincides with $J({\bf \xv})$ at ${\bf \xv}_k$ but otherwise upper-bounds $J({\bf \xv})$. We should choose
a majorizer $G_k(\xv)$ which can be minimized more easily (without having to solve a system of equations).
The $G_k(\xv)$ is defined as
\begin{equation}
G_k(\xv) = \|{\bf \yv - A\xv}\|_2^2 + (\xv -\xv_k)^{\top}(\eta {\boldsymbol I} - {\bf A}^{\top} {\bf A})(\xv -\xv_k),
\end{equation}
where ${\boldsymbol I}$ denotes the identity matrix and $\eta$ must be
chosen to be equal to or greater than the maximum eigenvalue of ${\bf A}^{\top}{\bf A}$. For the Hadamard sensing matrix used in our camera, the maximum eigenvalue of ${\bf A}^{\top}{\bf A}$ is easily obtained ($\eta\ge 1$). 
The update equation of $\xv_k$ is given by:
\begin{equation}\label{eq:ISTxk}
\xv_{k+1} = \xv_k +  \frac{1}{\eta} {\bf A}^{\top}(\yv - {\bf A} \xv_k).
\end{equation}
The GAP algorithm, proposed in~\cite{Liao14GAP}, which has been demonstrated high performance in video compressive sensing~\cite{Yuan14CVPR}, has the following update equation:
\begin{equation}\label{eq:GAPxk}
\xv_{k+1} = \xv_k +  {\bf A}^{\top}({\bf A A}^{\top})^{-1}(\yv - {\bf A} \xv_k)
\end{equation}
Under some condition of the sensing matrix ${\bf A}$, as the Hadamard matrix used in our system, $\bf A A^{\top}$ is the identity matrix and thus (\ref{eq:GAPxk}) is same as (\ref{eq:ISTxk}) with $\eta =1$.

Based on the above two methods, we propose a more general update equation for $\xv_k$:
\begin{equation}
\xv_{k+1} = \xv_k +  \xi{\bf A}^{\top}({\bf A A}^{\top})^{-1}(\yv - {\bf A} \xv_k),
\end{equation}
where $\xi$ is the step-size and we provide the step-size selection
with convergence guarantee in Theorem~\ref{thm:anytime}.

\subsubsection{Denoising}
\label{Sec:denoising}
Next, we consider the problem in (\ref{eq:alpha_L1}), which can be recognized as a denoising problem, and different from the previous algorithms~\cite{Figueiredo07MM}, the denoising is now performed on the patches of the image obtained by (\ref{eq:ISTxk}). Though various algorithms has been developed for this denoising algorithm and the dictionary learning approach has achieved state-of-the-art results, the computational cost is always high and thus not efficient.
On the other hand, the patches based transformation method~\cite{Dabov07BM3D} can provide excellent results efficiently.  
Inspired by this, and the success of JPEG compression, we here propose that, instead of learning a new dictionary, the DCT transformation will be used on the {\em overlapping} local patches for denoising.
We will demonstrate that, by using this and one more clustering step on these patches, we achieve better results than the advanced algorithm in~\cite{Mertzler14Denoising}\footnote{The DAMP algorithm~\cite{Mertzler14Denoising} may provide better results when the sensing matrix is Gaussian. However, here we focus on the Hadamard sensing matrix implemented in our system, which is realistic.}.
The efficiency of the transformation based denoising compared with the dictionary learning algorithm is the fast inverse transformation.
It is worth noting that when the transformation is performed on the onverlapping patches, the formulation in (\ref{eq:theta_l1}) does not fit anymore and thus we are using (\ref{eq:alpha_L1}) here, which can be reformulated as:
\begin{equation}\label{eq:xwTalpha}
\xv = {\bf W T \alphav}, 
\end{equation}
where $\Wmat$ is the average matrix for the same pixels in different overlapping patches and ${\bf T}$ is the inverse transformation matrix and ${\bf \alphav}$ is thus the vector of coefficients based on local patches.
Note that though ${\bf WT}$ is a fat matrix (more columns than rows), its pseudo-inverse can be computed efficiently; we can easily obtain $\xv$ from ${\bf \alphav}$ and vice verse.
This is the key that the proposed algorithm is efficient.

\subsubsection{Clustering Patches}
When better results are desired, we can achieve sparser representation of the image by clustering the patches into different groups, and in each group, we can perform a 3D transformation ({\em e.g.}, 2D DCT in space and a wavelet on the 3rd dimension), which can be represented as
\begin{eqnarray}\label{eq:xwT3alpha}
\xv^{(c)} = {\Wmat} ({\bf T}_{3}\otimes{\bf T}_{2} \otimes {\bf T}_{1} ) { \alphav}^{(c)}, 
\end{eqnarray}
where $^{(c)}$ denotes the $c$-th cluster and ${\bf T}_1, {\bf T}_2, {\bf T}_3$ symbolize the transformation bases in the first, second and third dimension, respectively.
${\alphav}^{(c)}$ is coefficients of patches in the $c$-th cluster
This clustering procedure can be implemented by $k$-means or block matching approaches.
Note that (\ref{eq:xwT3alpha}) is invertible; we can easily obtain $\xv^{(c)}$ from ${\bf \alphav}^{(c)}$ and vice verse.

The next step is to perform denoising (shrinkage) on ${\bf \alphav}$, which can be done using the soft-thresholding~\cite{Daubechie04IST}. However, how to select the threshold is always a problem and usually a cross-validation is required.
For the algorithm proposed here, the thresholding is performed on each cluster and thus is more challenging.
We propose an efficient way to select the threshold below.

\subsubsection{Determination of the Threshold}	
The GAP algorithm~\cite{Liao14GAP} enjoys the anytime property by using a particular way to threshold the coefficients. The basic idea to keep the non-zero coefficients as the same (or related) number of the measurement.
When the orthonormal transformation is used, the coefficients have the same dimension of ${\bf \xv}$. However, in our case, when the overlapping patches is used, there are far more coefficients than the dimension of ${\bf \xv}$, which is  implicitly represented by ${\bf \Wmat}$ in (\ref{eq:xwTalpha}).
Therefore, we extend the method in~\cite{Liao14GAP} by keeping the same compressive sensing ratio for each cluster in the coefficients.
Specifically, considering the compressive sensing ratio (CSr) defined by 
\begin{equation}
\text{CSr} = \frac{\text{number of row in } {\bf A}}{\text{number of column in } {\bf A}},
\end{equation}
we keep the non-zero number of coefficients in each cluster in proportion to (CSr$\times$the total number of coefficients in this cluster). We have found that this is very efficient in both simulation and real datasets.
This can also be seen as an adaptive threshold $\lambda^{(c)}_k$ imposed on the coefficients $\alphav^{(c)}$ for each cluster at every iteration:
\begin{eqnarray}
\betav^{(c)}_k &=& \alphav^{(c)}_k\odot \max\left\{1-\frac{\lambda^{(c)}_k}{|\alphav^{(c)}_k|},0\right\},
\end{eqnarray}
which is a shrinkage/thresholding operation~\cite{Beck09IST} and $\odot$ denotes the element-wise (Hadamard) product; $k$ symbolizes the iteration and $^{(c)}$ signifies the cluster number.
This implies that
\begin{eqnarray}
\beta^{(c)}_{k,i} &=&\left\{\begin{array}{lcc}
\alpha^{(c)}_{k,i}\left(1-\frac{\lambda_k^{(c)}}{|\alpha_{k,i}^{(c)}|}\right), & {\rm if}
& |\alpha_{k,i}^{(c)}|\ge \lambda_k^{(c)}, \\
0, &  &{\rm otherwise},
\end{array}\right.
\end{eqnarray}
where $\alpha^{(c)}_{k,i}$ is the $i$-th element of $\alphav^{(c)}_k, \forall i\in c$-th cluster.
The method described above provides an efficient way to select $\lambda_k^{(c)}$.

\subsection{Summary of SLOPE}
The optimization problem investigated here based on local overlapping patches in (\ref{eq:xwT3alpha}) as well as the global transformation based approach in (\ref{eq:theta_l1}) can be summarized as:
\begin{eqnarray} \label{eq:problem}
\min_{\bf \alphav} \|{\bf \alphav}\|_1, \quad {\text{subject to}} \quad \yv = \Amat {\cal H} \alphav,
\end{eqnarray}
where ${\cal H}$ symbolizes the transformation or basis and $\xv = {\cal H}\alphav$.
For the proposed SLOPE algorithm, the average matrix $\Wmat$ is also manifested in this ${\cal H}$. 

The proposed algorithm  can be summarized as an iterative two-step approach:
\begin{eqnarray}
\xv_{k+1} &=& \tilde{\xv}_k + \xi \Amat\ts (\Amat \Amat\ts)^{-1} (\yv - \Amat \tilde{\xv}_{k}), \label{eq:slope_x}\\
&\stackrel{\Amat \Amat\ts = \Imat}{=}&\tilde{\xv}_k + \xi \Amat\ts  (\yv - \Amat \tilde{\xv}_{k}), \label{eq:IST_x}\\
\betav^{(c)}_k &=& \alphav^{(c)}_k\odot \max\left\{1-\frac{\lambda^{(c)}_k}{|\alphav^{(c)}_k|},0\right\}, \label{eq:sh_alpha}
\end{eqnarray}
where 
\begin{itemize}
	\item $\alphav^{(c)}_k$ is obtained from $\xv_k$ via the transformation on overlapping patches in each cluster as shown in (\ref{eq:xwT3alpha}).
	\item $\tilde{\xv}_k$ is obtained from $\{\betav_k^{(c)}\}$ by transforming $\{\betav_k^{(c)}\}$  back to the image domain. 
	That is, $\tilde{\xv}_k$ is obtained from (\ref{eq:xwT3alpha}) by replacing $\alphav$ by $\betav$.
\end{itemize}

Since our algorithm is based on the shrinkage the coefficients of local overlapping patches, we term it as SLOPE (Shrinkage of Local Overlapping Patches Estimator), which is summarized in Algorithm~\ref{algo:slope}, and the `local' here denotes that the transformation is performed on local patches, rather than the global transformation performed on the entire image, {\em e.g.}, the wavelet.
\begin{center}
	\begin{algorithm}[htbp!]
		\caption{SLOPE}
		\begin{algorithmic}[1]
			\REQUIRE Measurements ${\yv}$, sensing matrix $\Amat$, and $\xi$.
			\STATE Initial $\xv_0 = \Amat\ts (\Amat \Amat\ts)^{-1} \yv$.
			\FOR{$k=1$ \TO Max-Iter }
			\STATE Update $\xv_k$ by Eq. (\ref{eq:slope_x}).
			\STATE Extract the overlapping patches from $\xv_k$.
			\STATE Cluster patches into $C$ clusters based on similarity (if necessary).
			\STATE Obtain the coefficients of each cluster $\{\alphav^{(c)}\}_{c=1}^C$ by imposing the transformation on the patches in the same cluster.
			\STATE Obtain the shrinked coefficients $\{\betav^{(c)}\}_{c=1}^C$ via the shrinkage/thresholding operation in (\ref{eq:sh_alpha}).
			\STATE Update $\tilde{\xv}_k$ by transforming $\{\betav^{(c)}\}_{c=1}^C$ back to the image (pixel) domain.
			\ENDFOR
		\end{algorithmic}
		\label{algo:slope}
	\end{algorithm}
\end{center}

When we write the optimization problem as in (\ref{eq:problem}), it seems the same as (\ref{eq:theta_l1}). However, significant difference exists when we solve them.
For the problem in (\ref{eq:problem}), existed algorithms usually solve the coefficients, $\thetav$, directly, instead of the desired signal $\xv$, as they assume there is one-by-one correspondence (each coefficient contributes equally to the signal) between $\thetav$ and $\xv$ (since wavelet is usually used).
However, when the overlapping patches are used, it is different to update $\xv$ as in (\ref{eq:slope_x}) from updating $\thetav$, as each coefficient is weighted differently (one pixel corresponds to several different coefficients). 
We unveil the difference below.
Consider updating $\thetav$ directly, and the solution to the first step (\ref{eq:Jx}) will be
\begin{eqnarray}\label{eq:thetav_k+1}
\thetav_{k+1} &=& \thetav_k + \xi \Rmat\ts (\Rmat \Rmat\ts)^{-1} (\yv - \Rmat {\thetav}_{k}).
\end{eqnarray} 
Note that $\Rmat$ is a matrix including the sensing matrix $\Amat$, the pixel averaging matrix $\Wmat$ and the transformation matrix (or the dictionary); $\Rmat = \Amat \Wmat \Tmat$. This leads to that $\Rmat \Rmat\ts$ is not an identity matrix and it is not easy to explicitly write this matrix. Therefore, updating $\xv$ directly as in (\ref{eq:slope_x}) is more straightforward and saves a lot of computational workload as well as memory.
On the other hand, because of the overlapping patches, the coefficients are not equally important and $\Wmat$ imposes weights for each coefficient.
Left-multiplying $\Wmat\Tmat$ on (\ref{eq:thetav_k+1}) will lead to (\ref{eq:slope_x}):
\begin{eqnarray}
\Wmat\Tmat \thetav_{k+1} &=& \Wmat\Tmat \thetav_k + \xi \Wmat\Tmat  (\Amat \Wmat \Tmat)\ts ((\Amat \Wmat \Tmat) (\Amat \Wmat \Tmat)\ts)^{-1} (\yv - \Amat \Wmat \Tmat {\thetav}_{k}), \\
\xv_{k+1} &=& \xv_k + \xi \Wmat\Wmat\ts (\Wmat\Wmat\ts)^{-1}  \Amat\ts  (\yv - \xv_{k}) =  \xv_k + \xi  \Amat\ts  (\yv - \Amat\xv_{k})\label{eq:xk+1_Tw},
\end{eqnarray}
where we have used $\Tmat\Tmat\ts = \Imat$ and $\Amat\Amat\ts = \Imat$ and note that
$\Wmat\Wmat\ts $ is a diagonal matrix~\cite{Dong14TIP} with each diagonal element in $(0,1]$.

If $(\Rmat \Rmat\ts)^{-1} $ is not used as in (\ref{eq:thetav_k+1}), it will be
\begin{eqnarray}\label{eq:thetav_k1}
\thetav_{k+1} &=& \thetav_k + \xi \Rmat\ts  (\yv - \Rmat {\thetav}_{k}).
\end{eqnarray} 
This will bias the solution of $\thetav$, since it treats each coefficient equally~\cite{Bioucas-Dias2007TwIST}. After several iterations, the error will be accumulated and therefore the results are not as good as updating $\xv$ directly.
To see this explicitly, left-multiplying $\Wmat\Tmat$ on both sides of (\ref{eq:thetav_k1}), we have
\begin{eqnarray}
\Wmat\Tmat \thetav_{k+1} &=& \Wmat\Tmat \thetav_k + \xi \Wmat\Tmat  (\Amat \Wmat \Tmat)\ts  (\yv - \Amat \Wmat \Tmat {\thetav}_{k}). \\
\xv_{k+1} &=& \xv_k + \xi \Wmat\Wmat\ts  \Amat\ts  (\yv - \Amat\xv_{k})  \label{eq:xk+1_w},
\end{eqnarray}
where we have used $\Tmat\Tmat\ts = \Imat$ and we can see from (\ref{eq:xk+1_w}) that since $\Wmat \Wmat\ts$ is not an identity matrix and thus it is different from (\ref{eq:IST_x}).
$\Wmat$ plays the role of weighting each coefficient for the overlapping patches, since each pixel belongs to different patches. When the coefficients $\thetav$ are transformed back to pixels, each pixel should have equal weight (since they are equally important) and therefore, $(\Rmat \Rmat\ts)^{-1}$ in (\ref{eq:thetav_k+1}) balances this importance.
However, when (\ref{eq:thetav_k1}) is used, each pixel will have a different weight as in (\ref{eq:xk+1_w}) (because of $\Wmat\Wmat\ts$) and the reconstruction error is thus introduced.

\subsection{Convergence of SLOPE}
We now prove that under the lensless compressive imaging case considered in our system, $\Amat\Amat\ts = \Imat$, SPLOE is an anytime algorithm; the solution sequence $\{\xv_k\}_{k=1}^\infty$ is monotonically converging to the true signal, by selecting the proper $\lambda_k$ at each iteration with a certain range of $\xi$.

Consider the true image is $\xv^*$ and 
\begin{equation}
\xv^* = {\cal H}\alphav^*,
\end{equation}
with $\alphav^*$ denoting the (true) sparse coefficients in the transformation domain. 
We need the following conditions to prove the anytime property of SLOPE:
\begin{enumerate}
	\item [a)] Initialization with 
	\begin{eqnarray}
	\xv_0 &=& \Amat\ts (\Amat \Amat\ts)^{-1} \yv \stackrel{\Amat \Amat\ts = \Imat}{=} \Amat\ts \yv \label{eq:x0};
	\end{eqnarray}
	\item[b)]
	For $k$-th iteration, select $\lambda_k$ such that
	\begin{equation}
	\|\betav_k\|_1 \ge \|\alphav^*\|_1.  \label{eq:ell1_ball}
	\end{equation}
	where $\|\cdot\|_1$ denotes the $\ell_1$-norm, the summation of absolute values of each entry.
	\item[c)]
	Consider that in each iteration, $\lambda_k$ is selected to keep at most $m^*_{\lambda_k}$ nonzero elements in $\betav_k$, with $m^*_{\lambda_k}< N_c$, where $N_c$ is the number of coefficients (the dimension of $\alphav$ or $\betav$; it is much larger than $N$, the dimension of $\xv$, when the overlapping patch is used).
	We need the RIP (restricted isometry property) condition~\cite{cs_Candes06,Candes05,Candes06ITT} on $\Rmat = \Amat \Wmat\Tmat$ such that
	\begin{eqnarray}\label{eq:rip_m*}
	0<\delta_{m^*_{\lambda_k} + K^*} <1,
	\end{eqnarray}
	where $K^*$ is the number of non-zero elements in $\alphav^*$ and $m^*_{\lambda_k} \ge K^*$.
\end{enumerate}
To see the intuition behind (\ref{eq:x0})-(\ref{eq:ell1_ball}), we consider $\yv = \Amat \xv^*$ as a line (manifold) and $\xv_0$ is initialized by touching the line with a large $\ell_1$ ball (formed by $\alphav$) (because $\yv = \Amat \xv_0$). At each iteration, we shrink the $\ell_1$-ball by the operation in (\ref{eq:sh_alpha}) but with the condition that this $\ell_1$-ball formed by $\betav_k$ is larger than the true $\ell_1$-ball formed by $\alphav^*$.
Eventually, in the ideal case, the $\ell_1$-ball formed by $\betav_k$ will be the same as the true $\ell_1$-ball formed by $\alphav^*$; thus $\xv_k$ recovers the true signal $\xv^*$.
From the initialization in (\ref{eq:x0}), we have $\yv = \Amat \xv_0 = \Amat {\cal H} \alphav_0$. Since $\yv = \Amat \xv_0 = \Amat {\cal H} \alphav^*$ and $
\alphav^*$ is a minimizer, we have $\|\alphav_0\|_1 \ge \|\alphav^*\|_1$, and thus, there exists $\lambda_0$ such that $\|\betav_0\|_1 \ge \|\alphav^*\|_1$.
The condition in (\ref{eq:ell1_ball}) can be relaxed to that, for each step, $\|\alphav_k\|_1 \ge \|\alphav^*\|_1$, so that $\lambda_k$ exists to hold the anytime convergence of SLOPE. Please refer to~\cite{WangAIT15} for an alternative interpretation.


\begin{theorem}	
	\label{thm:anytime}
	Let $\Rmat$ satisfy RIP, {\em i.e.}, there exists $0<S^*\le N_c$ and $0<\delta_{S^*} <1$ such that
	\begin{eqnarray}
	(1-\delta_{S^*})\|\alphav_s\|_2^2 \le \|\Rmat_s \alphav_s\|_2^2 \le (1+\delta_{S^*})\|\alphav_s\|_2^2,
	\end{eqnarray}
	for all $0<s \le S^*$, where $\alphav_s \in{\mathbb R}^{s}$, $\Rmat_s\in{\mathbb R}^{M\times s}$ is formed from columns of $\Rmat$. 
	Let $\alphav^*$ be the true solution of (\ref{eq:problem}) and its sparsity is $K^*$.
	Let $m^*_{\lambda_k}$ be the sparsity of $\betav_k$ given by (\ref{eq:sh_alpha}).
	If there exists a sequence $\lambda_k>0$, such that
	\begin{eqnarray}
	\|\betav_k\|_1 &\ge& \|\alphav^*\|_1, 
	\quad {\rm and}\quad m^*_{\lambda_k} + K^* \le S^*,
	\end{eqnarray}
	then $\alphav_k$ from Algorithm~\ref{algo:slope} (initialized by (\ref{eq:x0}))  monotonically converges to $\alphav^*$ for all $\xi \in(0,2)$; Algorithm~\ref{algo:slope} is an anytime algorithm.
	%
\end{theorem}
\begin{proof}
	The full proof is presented in the Appendix.
	Here we review the main steps. Recall that the true signal is $\xv^*$ and the algorithm provides a sequence of solution $\xv_k$ at each iteration. We prove in the Appendix that:
	\begin{itemize}
		\item[1)] when $\xi\in(0,2]$, $\|\xv_k -\xv^*\|_2^2$ ($\|\alphav_k -\alphav^*\|_2^2$) monotonically non-increases;
		\item[2)] when $\xi\in(0,2)$, $\|\xv_k -\xv^*\|_2^2$ ($\|\alphav_k -\alphav^*\|_2^2$) monotonically decreases and  converges to a constant. 
		\item[3)] when $\xi\in(0,2)$, with the RIP condition on $\Rmat$, $\|\xv_k -\xv^*\|_2^2$ ($\|\alphav_k -\alphav^*\|_2^2$) monotonically converges to zero. 
	\end{itemize}
\end{proof}

It is worth noting that the algorithm in (\ref{eq:slope_x})-(\ref{eq:sh_alpha}) has the same formulation of the iterative shrinkage/thresholding algorithm (ISTA)~\cite{Daubechie04IST,Beck09IST} under the condition $\Amat\Amat\ts = \Imat$.
However, we have proved that when the step size $\xi\in(0,2)$, ISTA is an {\em anytime} algorithm if the initialization and thresholds are selected as mentioned in our algorithm.
\begin{corollary}
	When $\Amat\Amat\ts = \Imat$, the ISTA used to solve the optimization problem in (\ref{eq:problem})  is an anytime algorithm if it is initialized  using (\ref{eq:x0}) and the threshold for each step satisfies the same condition as SLOPE with the step size $\xi\in(0, 2)$.
\end{corollary}
\begin{proof}
	The proof follows Theorem~\ref{thm:anytime}.
\end{proof}

{\em Remarks:}
\begin{itemize}
	\item Unlike the ISTA algorithm derived from the MM approach in (\ref{eq:ISTxk}), which needs $\eta \ge \max{\rm eig}(\Amat\ts \Amat)$ (which equals the step size $\xi = \frac{1}{\eta}$), we only need the step size $\xi\in (0,2)$. Usually, $\xi\ge1$ is used for fast convergence. 
	\item The GAP algorithm proposed in~\cite{Liao14GAP} can be seen as a special case of our algorithm with $\xi =1$ as in (\ref{eq:GAPxk}). On the other hand, GAP is developed based on weighted group $\ell_{2,1}$ norm, and it does not require $\Amat\Amat\ts = \Imat$.\footnote{Our proof can also be extended to the case which does not need $\Amat\Amat\ts = \Imat$.}
	\item 
	In order to select the appropriate $\lambda_k$ at each iteration, the method proposed in~\cite{Liao14GAP} can still be used.
	\begin{eqnarray}
	\tilde{\alphav}_k &=& {\rm sort}(|\alphav_k|, {\text {`descend'}}),\\
	\lambda_k &=& \tilde{\alpha}_{k,m^*+1}, \label{eq:def_m_star}
	\end{eqnarray}
	where $\tilde{\alpha}_{k,m^*+1}$ is the $(m^*+1)$-th entry of $\tilde{\alphav}_k $, which sorts the absolute values of $\alphav_k$ from large to small. Then the sparsity of $\betav_k$ generated by Algorithm~\ref{algo:slope} will be $m^* = m^*_{\lambda_k}$ in Theorem~\ref{thm:anytime}. We found in the experiments that setting $m^*\in [0.5N_c, N_c-1]$ always provides good results.
	Similar selection approach can also be found in \cite{WangAIT15}, where the generalized-RIP is introduced and the adaptively iterative thresholding algorithm are proved to be converged linearly. 
	\item 
	Equation (\ref{eq:ell1_ball}) is a sufficient (thus restricted) condition. Even we select a larger $\lambda_k$, the algorithm may still converge well. For instance, if we can select $\lambda_k$ such that the support of $\betav_k$ (${\cal J}_+$ in (\ref{eq:J+})) includes the support of $\alphav^*$ (${\cal I}_+$ in (\ref{eq:I+})) in each iteration, SLOPE will monotonically converge to zero. 
	\item SLOPE explores the sparsity of local patches, while conventional compressive sensing inversion algorithms are often developed based on the sparsity of wavelet coefficients. However, the wavelet coefficients are usually not sparse, but compressible~\cite{Yuan14TSP}. 
	On the other hand, the DCT coefficients for overlapping patches are sparser than the wavelet coefficients, as the number of coefficients $N_c$ is much larger than the number of wavelet coefficients (recall that $\Rmat \in {\mathbb R}^{M\times N_c}$). Similar case exists for the patch-based dictionary learning model~\cite{Elad06TIP} where the sparsity is imposed on coefficients.  
	When the patch is small, usually, only a DC coefficient is sufficient to represent a single local patch. Therefore, it is more reasonable to define the sparse level $K$ in (\ref{eq:rip_m*}) on the coefficients of overlapping patches.
	Similarly, since $N_c\gg N$, the selection of $m^*$ ($\lambda_k$) has a large degree of freedom.
	In addition, our theorem is not limited by the local patch based model; it also fits the wavelet transformation based algorithms.
	\item In the noisy case, SLOPE can also be used. It is worth noting that we only impose that $\betav_k$ is sparse, rather than $\alphav_k$. Therefore, the different between $\Amat\xv_k$ and $\Amat\tilde{\xv}_k$ can provide a good estimate of the (measurement) noise~\cite{Liao14GAP}.
	Experimental results on real data in Section~\ref{Sec:Hardware} verify the robustness of SLOPE under the noisy case.
\end{itemize}
\subsection{Relation to ADMM}
The Alternating Direction Method of Multipliers (ADMM) algorithm~\cite{ADMM2011Boyd} provides an alternative solution to a lot of optimization problems.
When the ADMM is utilized in our problem, the difference of ADMM compared with IST, GAP and SLOPE lies in how to update $\xv$ as stated in Section~\ref{Sec:update_x}.
Under the ADMM formulation, introducing regulizers $\{b,c\}$, the cost function of (\ref{eq:problem}) is:
\begin{align}
{\cal L}(\xv,\tilde{\xv},\alphav,b,c) &= \frac{1}{2}\|\yv-\Amat\xv\|_2^2 + \frac{b}{2}\|\xv-\tilde{\xv}\|_2^2 + c \|\alphav\|_1  \qquad {\text { with }}(\tilde{\xv} = {\cal H}\alphav).
\end{align}
ADMM cyclically solves the following subproblems:
\begin{eqnarray}
\xv_{k+1}&: =& \arg \min_{\xv} \frac{1}{2}\|\yv-\Amat\xv\|_2^2 + \frac{b}{2}\|\xv-\tilde{\xv}_k\|_2^2 , \label{eq:admm_x_k+1}\\ 
\tilde{\xv}_{k+1} &:= & \arg\min_{\tilde{\xv}} \frac{b}{2}\|\xv_{k+1}-\tilde{\xv}\|_2^2 + c \|\alphav\|_1. \label{eq:admm_theta_k+1}
\end{eqnarray}
While (\ref{eq:admm_theta_k+1}) can be solved using the same shrinkage/thresholding approach as described in Section~\ref{Sec:denoising}, we here focus on the update of $\xv$, to solve (\ref{eq:admm_x_k+1}).
Given $\tilde{\xv}_k$, (\ref{eq:admm_x_k+1}) is a quadratic optimization problem and $\xv$ can be simplified to:
\begin{eqnarray}
(\Amat\ts\Amat + b\Imat)\xv = \Amat\ts \yv + b \tilde{\xv},
\end{eqnarray}
which admits the following closed-form solution:
\begin{eqnarray}
\xv_{k+1}&=& (\Amat\ts\Amat + b\Imat)^{-1} (\Amat\ts \yv + b\tilde{\xv}_k).
\end{eqnarray}
Since in the CS framework, $\Amat$ is a fat matrix, $(\Amat\ts\Amat + b\Imat)$ will be a large matrix and thus the matrix inversion formula can be used to simplify the problem:
\begin{align} \label{eq:admm_xv_inv}
\xv_{k+1} = \left[b^{-1}{\Imat}-b^{-1}\Amat^{\top}(\Imat + \Amat b^{-1}\Amat^{\top})^{-1}\Amat b^{-1}\right][\Amat^{\top}\yv + b \tilde{\xv}_k],
\end{align}
In our case considered in the real system  $\Amat\Amat\ts = \Imat$, 
\begin{eqnarray}\label{eq:admm_xv_inv_sim}
\xv_{k+1} &= & \tilde{\xv}_k + \frac{\Amat\ts (\yv- \Amat \tilde{\xv}_k)}{b + 1},
\end{eqnarray}
which is same as (\ref{eq:GAPxk}) if $b = 0$ ($\xi = 1$).
Comparing the update equation of $\xv$ in (\ref{eq:admm_xv_inv_sim}) of ADMM with the update rule of our SLOPE algorithm in (\ref{eq:IST_x}), we observe that:
\begin{itemize}
	\item The ADMM formulation of updating $\xv$ is a special case of SLOPE with $\xi = \frac{1}{1+b}$.
\end{itemize}
Since $b$ is usually selected to be a small number, the ADMM update rule is very similar to the case $\xi  =1$. Therefore, the anytime property of SLOPE still holds if the ADMM updating rule is adopted.

\begin{figure}[htbp!]
	\centering
	\includegraphics[width=0.8\textwidth]{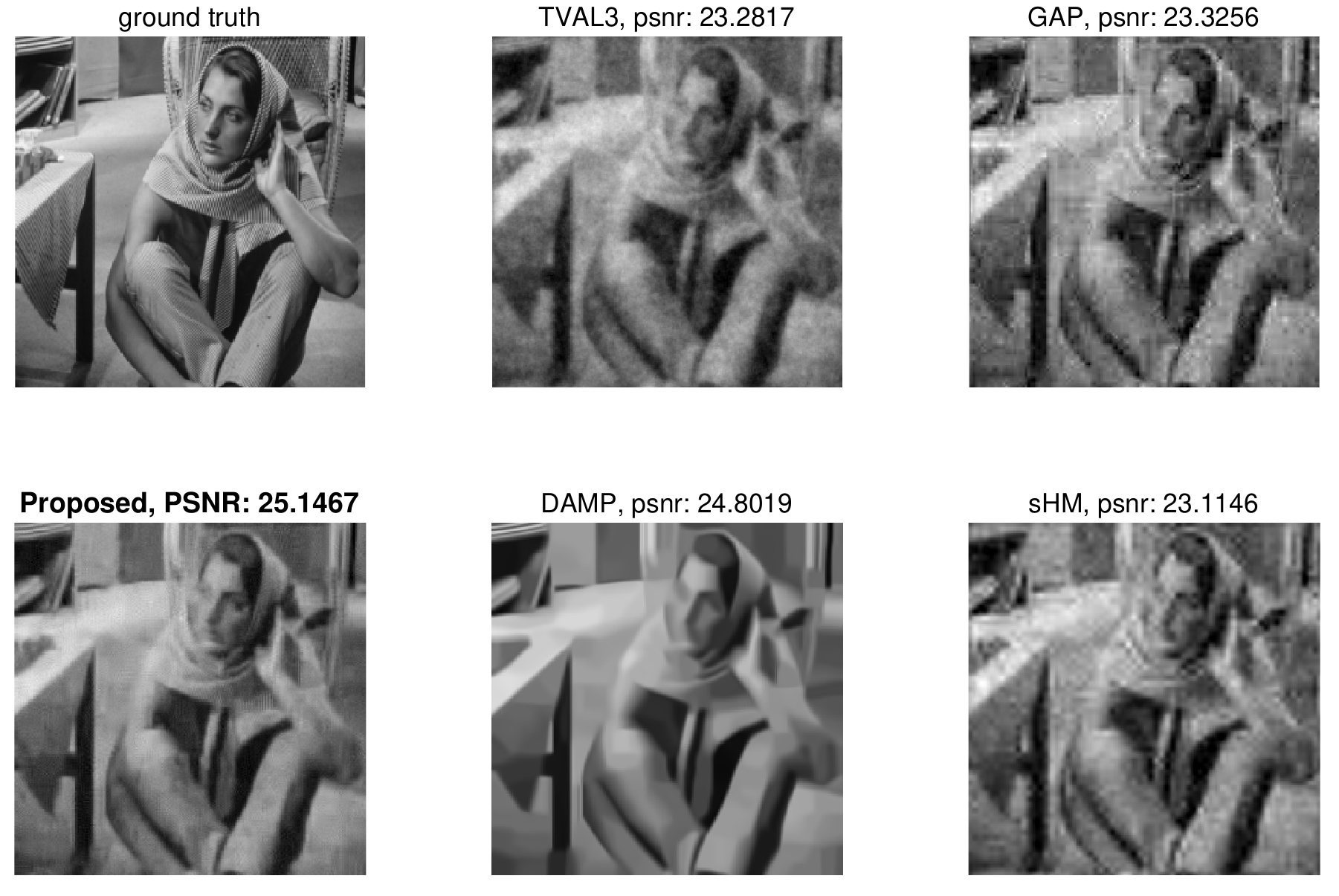}
	\vspace{-3mm}
	\caption{Simulation: reconstruction results of different algorithms at CSr$=0.1$, image size $256\times 256$.}
	\label{fig:baba01}
\end{figure}

\begin{table}[htbp!]
	\caption{Simulation: reconstruction PSNR (dB) of different images with diverse algorithms at various CSr. SLOPE is the proposed algorithm.}
	\centering
	\begin{tabular}{|c|c|c|c|c|c|c|}
		\hline Image & CSr
		&  TVAL3  &  GAP & sHM & DAMP & SLOPE \\
		\hline \hline
		\multirow{5}{*}{\includegraphics[width=1.5cm, height=1.5cm]{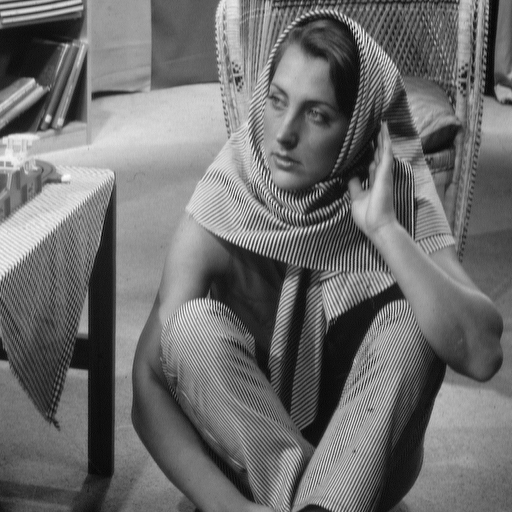}} &	0.02 & 16.91 & 19.19 & 8.28 & 18.47 & {\bf 20.06}\\
		& 0.04 & 19.37 & 20.99 & 18.37 & 20.44  & {\bf 21.72}\\
		& 0.06 & 21.48  & 22.25 & 20.90 & 22.27   &  {\bf 22.69}\\
		& 0.08 & 22.84 & 23.08 & 21.46  & 23.85  & {\bf 24.42} \\
		& 0.1& 23.42  & 23.74 & 23.34  & 25.04  & {\bf 25.15}\\
		\hline\hline
		\multirow{5}{*}{\includegraphics[width=1.5cm, height=1.5cm]{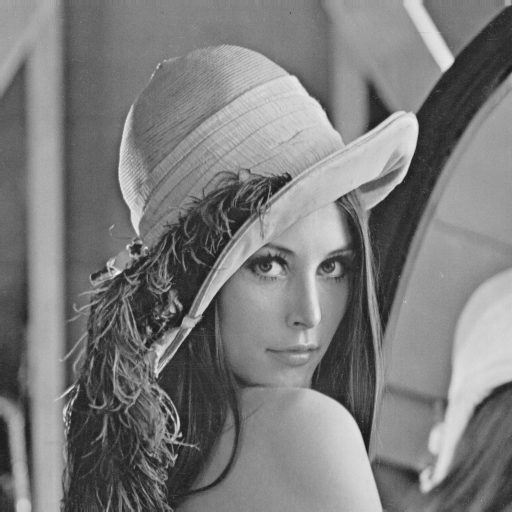}} & 0.02 & 17.39 & 20.06 & 9.28 & 19.40 & {\bf 20.85}\\
		& 0.04 & 20.01 & 21.73 & 19.24 & 21.98 & {\bf 22.63}\\
		&  0.06 & 22.23  & 22.97& 21.59 & 23.89  &  {\bf 24.76}\\
		&  0.08 & 22.76 & 23.88 & 22.47  & 25.17  & {\bf 25.60} \\
		&  0.1& 23.79  & 24.61 & 24.13  & 26.19  & {\bf 26.37}\\
		\hline\hline
		\multirow{5}{*}{\includegraphics[width=1.5cm, height=1.5cm]{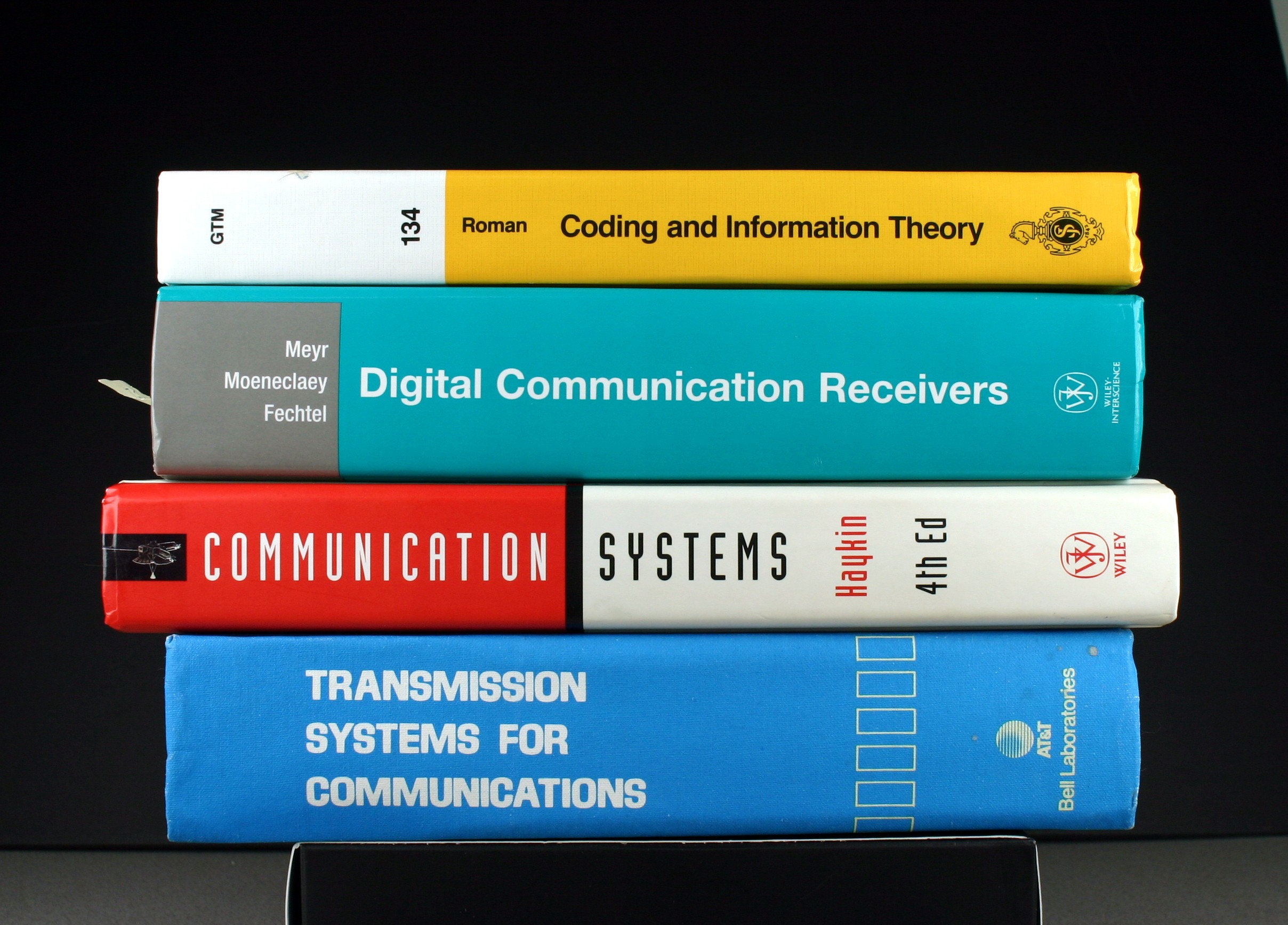}} &  0.02 & 14.99 & 17.70 & 9.85 & 15.97 & {\bf 18.37}\\
		&  0.04 & 17.04 & 19.84 & 17.23 & 18.99 & {\bf 20.45}\\
		&  0.06 & 18.99  & 21.09 & 18.98 & 21.83 &  {\bf 23.66}\\
		&  0.08 & 20.12 & 21.98 & 20.80  & 23.80  & {\bf 24.55} \\
		&  0.1& 20.82  & 22.54 & 21.44  & 25.18 & {\bf 26.41}\\
		\hline\hline
		\multirow{5}{*}{\includegraphics[width=1.5cm, height=1.5cm]{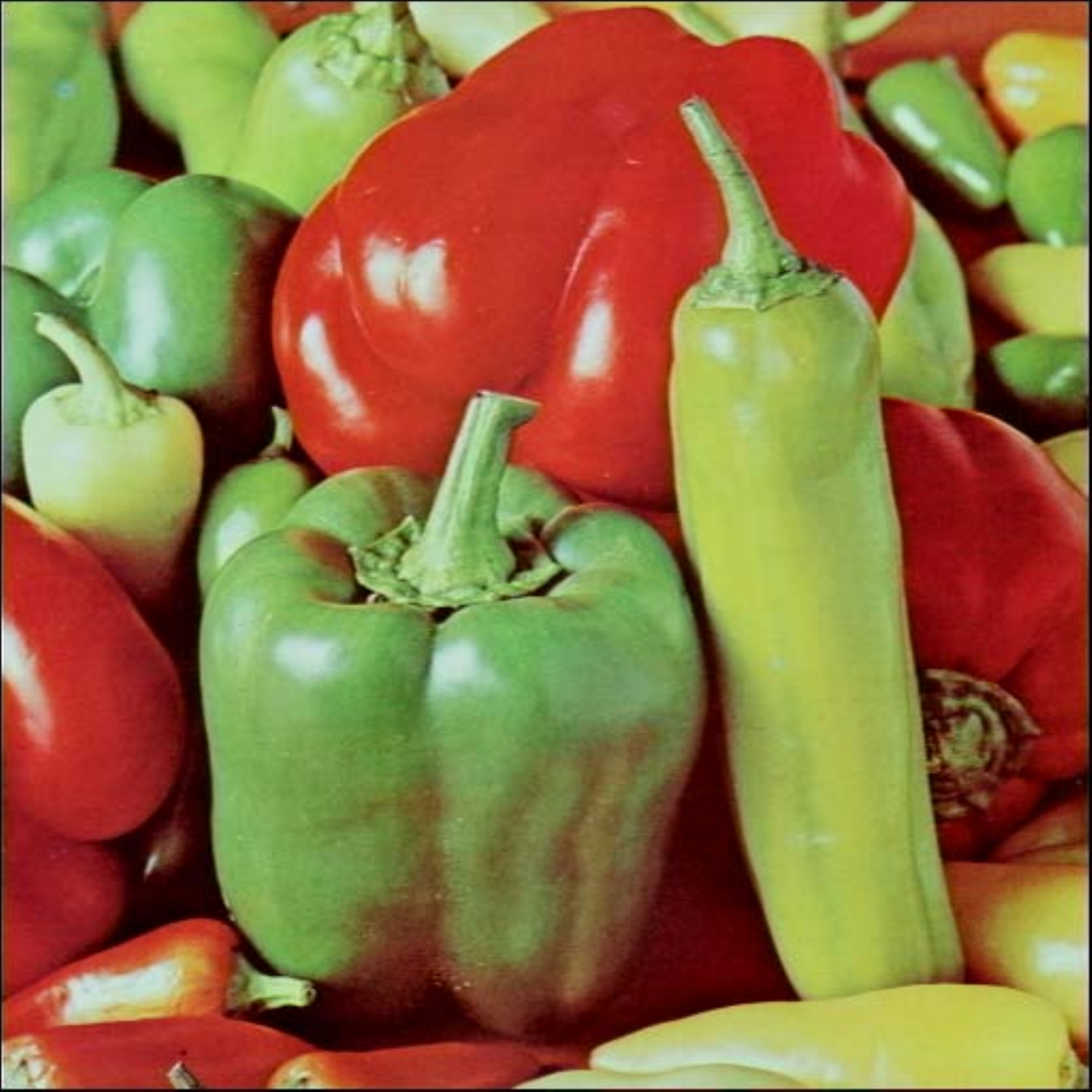}} &	0.02 & 16.80 & 18.83 & 9.70 & 17.82 & {\bf 20.00}\\
		&  0.04 & 19.44 & 20.85 & 18.60 & 20.21 & {\bf 21.85}\\
		&  0.06 & 21.27  & 22.25 & 20.83 & 22.10 &  {\bf 23.07}\\
		&  0.08 & 22.53 & 23.31 & 22.71  & 23.69 & {\bf 24.03} \\
		&  0.1& 23.13  & 24.19 & 23.81  & 25.33 & {\bf 25.88}\\
		\hline
	\end{tabular}
	\label{Table:sim_PSNR}
\end{table}	
\begin{figure}[htbp!]
	\centering
	\includegraphics[width=0.8\textwidth]{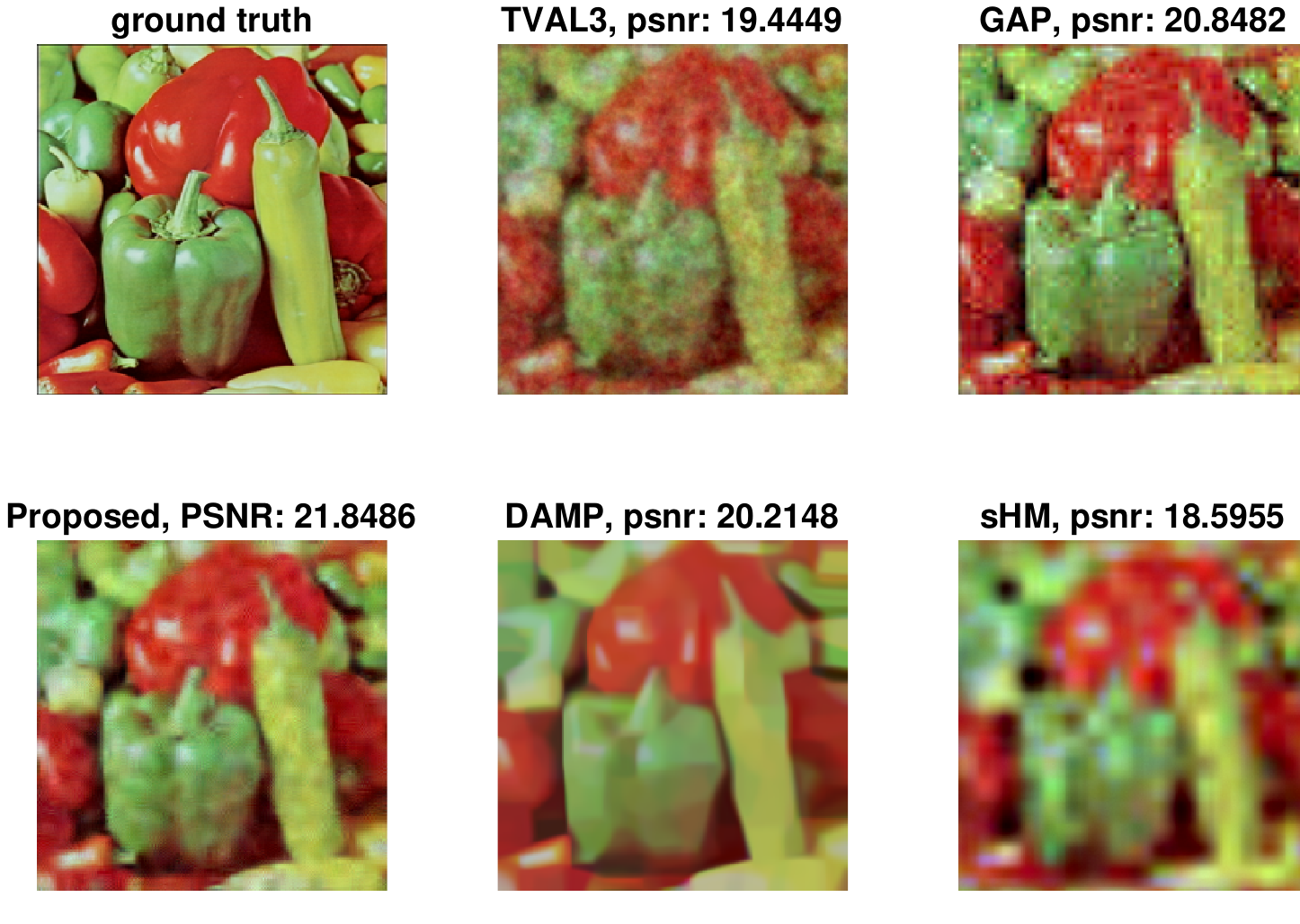}
	\vspace{-3mm}
	\caption{Simulation: reconstruction results with different algorithms at CSr$=0.04$ ($4\%$ of the total pixel number, $256\times 256 \times 3$). }
	\label{fig:peppers04}
\end{figure}
\section{Simulation Results}
\label{Sec:sim}
To verify the performance of the proposed algorithm, we conduct SLOPE on some simulation datasets, which is summarized in Table~\ref{Table:sim_PSNR}.
Different from the simulation conducted in previous papers~\cite{Mertzler14Denoising}, the sensing matrix used in our work is the permuted Hadamard matrix as implemented in our hardware. Therefore, the reported results may be different from them in other papers. 
The proposed SLOPE algorithm is compared with the following four algorithms:
1) TVAL3~\cite{Li13COA}, 2) GAP based on wavelet~\cite{Liao14GAP}, 3) DAMP~\cite{Mertzler14Denoising} with BM3D denoising, and 4) sHM by exploiting the tree structure in wavelet~\cite{Yuan14TSP}.
Since when CSr$=0.1$, very good results have been achieved for most images (Figure~\ref{fig:baba01}), we here spend more efforts on the extremely low CSr, in particular CSr$<0.1$.
For all the simulated images used here, we resize them to size $256\times 256$.
For the RGB image, we use R, G, and B sensors to sample each channel separately.
One example is shown in Figure~\ref{fig:peppers04}. 
It can be observed that DAMP over-smooths the image while the proposed algorithm reserves more details. Different types of artifacts exist in other algorithms.
Regarding the computation time, for each iteration, our algorithm takes about 0.28 seconds (at CSr = 0.1), which is similar to TVAL3 and GAP, and we found that 50 iterations are sufficient to present decent results. One iteration in DAMP takes longer than our algorithm, about 2.83 seconds for the $256 \times 256 \times 3$ RGB image.
The patch size of $8\times 8$ are used for all experiments and $\xi = 1.5$ is employed as the step size.
%
%
%
%
%
\begin{figure}[htbp!]
	\centering
	\includegraphics[width=0.7\textwidth]{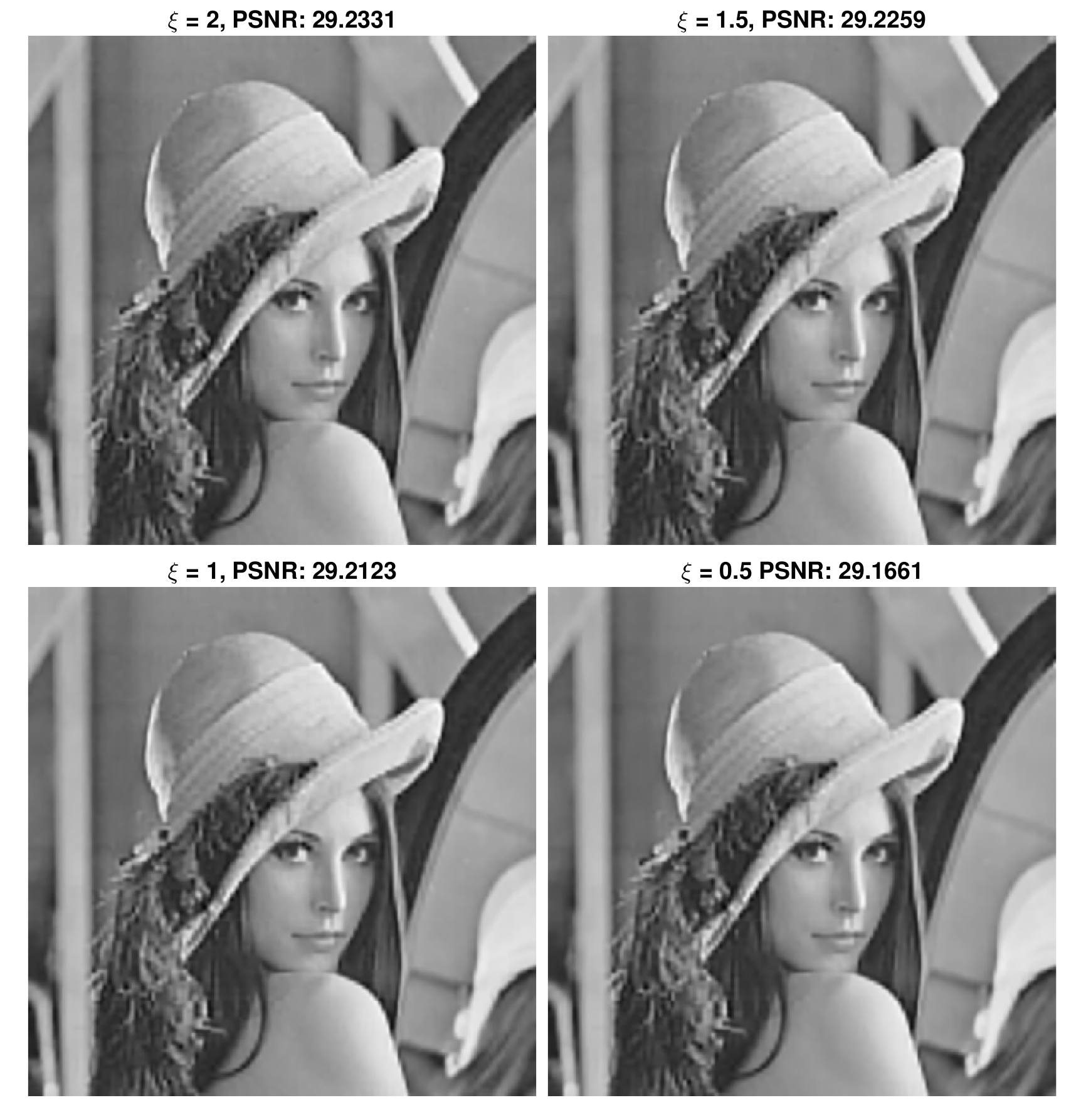}
	\vspace{-3mm}
	\caption{Simulation: reconstruction images with different step size $\xi$, for the $512\times 512$ image. Results are obtained via running the proposed algorithm 100 iterations. CSr = 0.05.}
	\label{fig:Lena_rec}
\end{figure}
\begin{figure}[htbp!]
	\centering
	\includegraphics[width=0.48\textwidth]{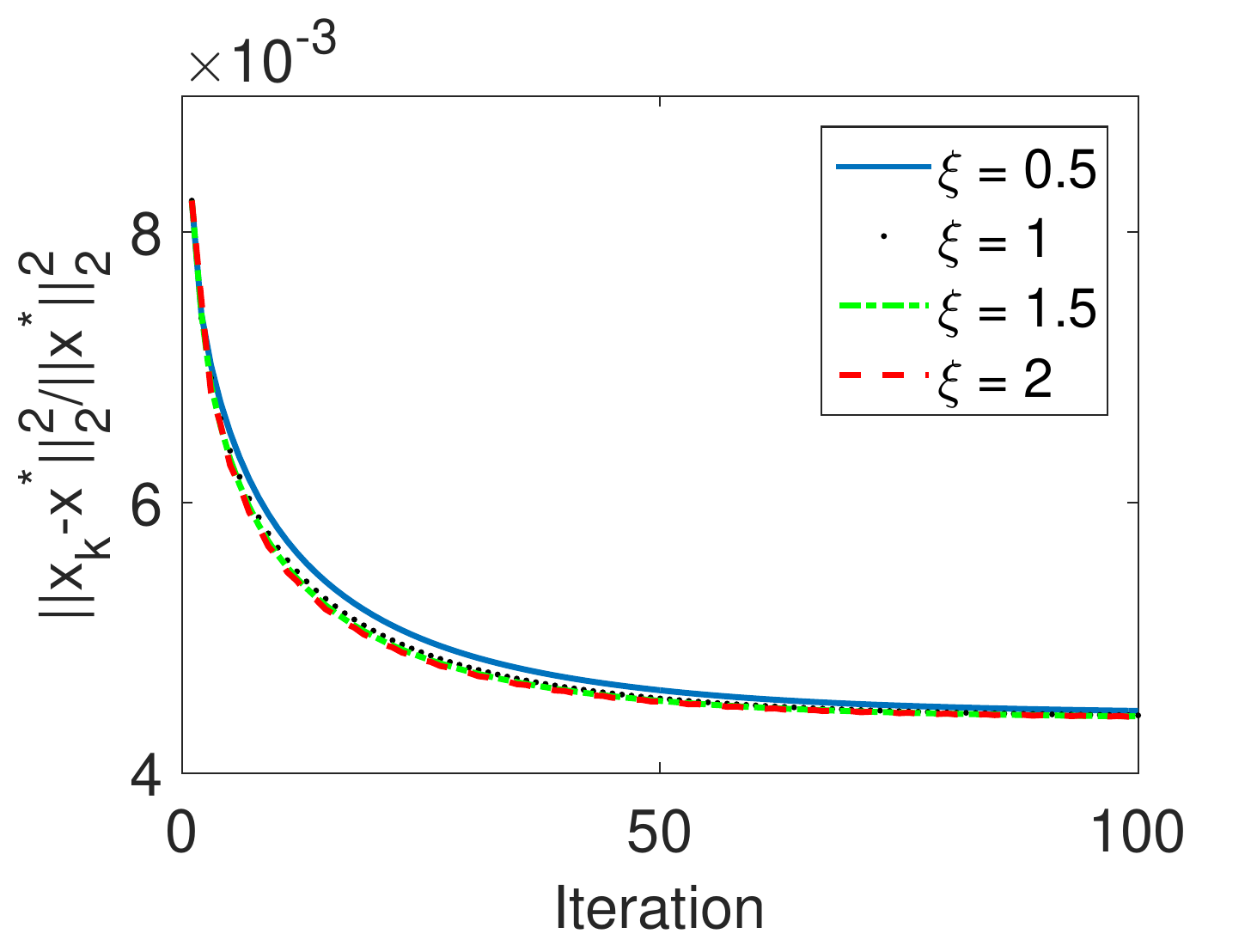}~
	\includegraphics[width=0.48\textwidth]{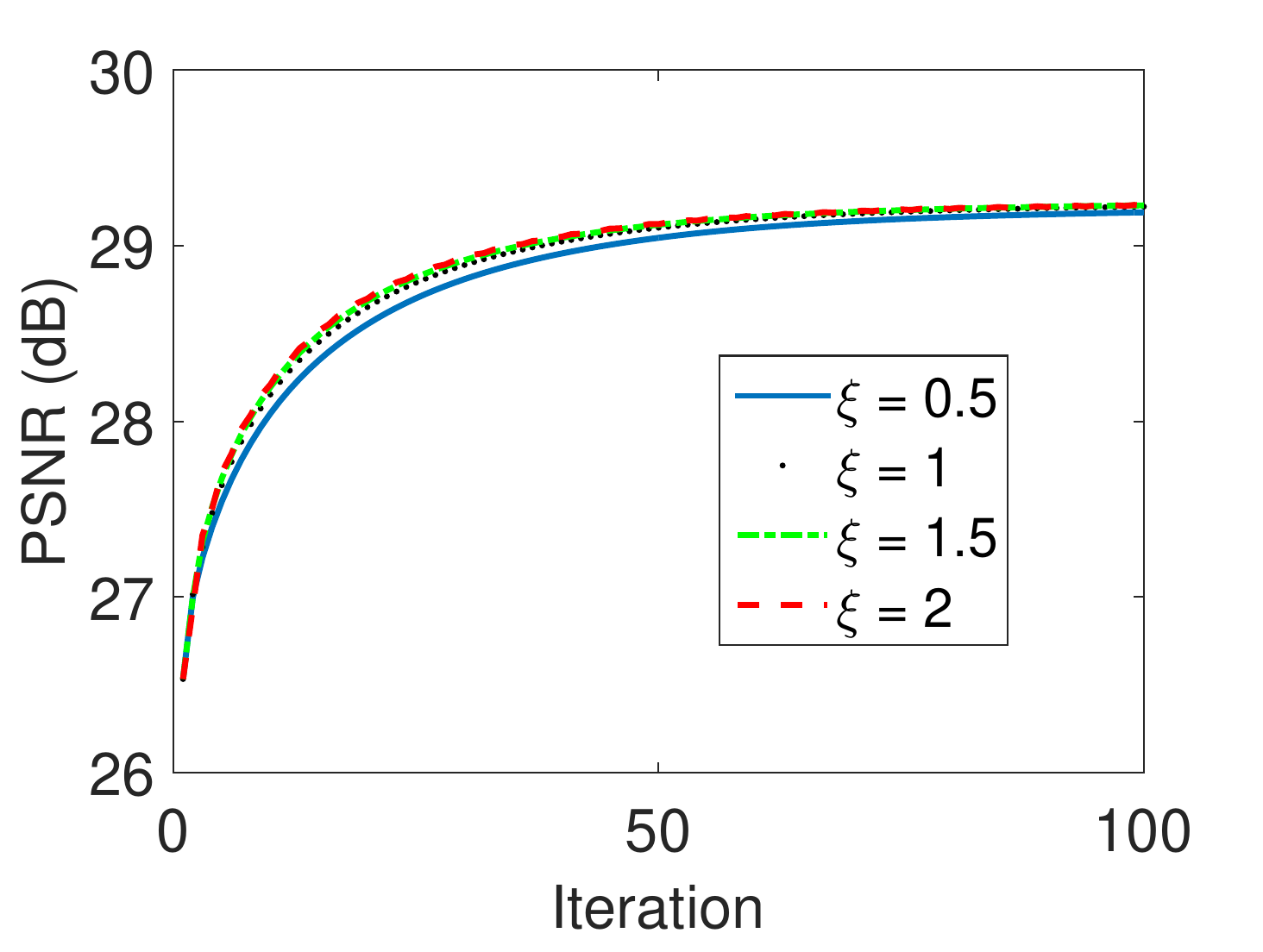}
	\vspace{-3mm}
	\caption{Reconstruction error (left) and PSNR (right) with different step size $\xi$, for the $512\times 512$ image. CSr = 0.05.}
	\label{fig:Lena_psnr}
\end{figure}
\subsection{Anytime Verification}
We next verify the anytime property of the proposed SLOPE algorithm by considering the ``Lena" image of size $512\times 512$. Similarly, the permuted Hadamard matrix is used.
We test four values of the step size $\xi = \{0.5, 1, 1.5, 2\}$ by setting CSr = 0.05. 
The results are shown in Figure~\ref{fig:Lena_rec}.
The reconstruction errors and PSNRs at each iteration with different step-size are plotted in Figure~\ref{fig:Lena_psnr}. It can be seen that the reconstruction errors are decreasing monotonically for each iteration while the PSNRs are increasing for each iteration (especially when $\xi = \{0.5, 1, 1.5\}$). 
Furthermore, we observe that a larger $\xi$ leads to faster convergence.
In addition, a larger $\xi$ usually needs a larger $m^*$ to select the $\lambda_k$, thus to ensure the anytime convergence of SLOPE.
We have observed in our experiments that when $\xi = 2$, the PSNR of the reconstructed image sometimes does not increase monotonically, which is consistent with the range of step-size in Theorem~\ref{thm:anytime}. 

\begin{figure}[htbp!]
	\centering
	\includegraphics[width=0.8\textwidth]{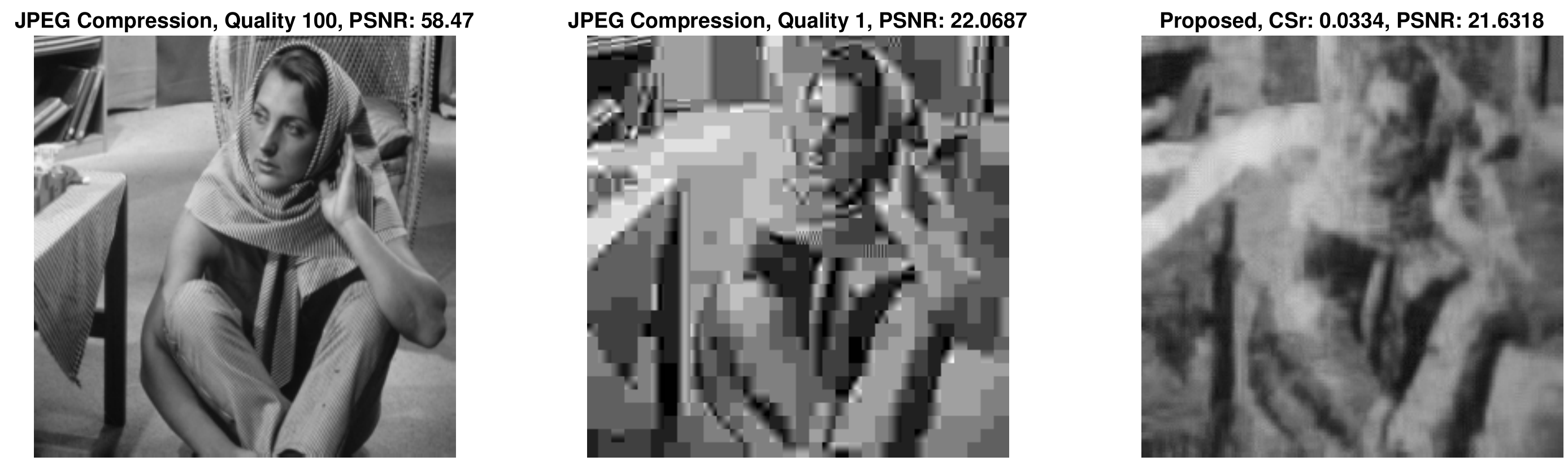}\\
	\includegraphics[width=0.8\textwidth]{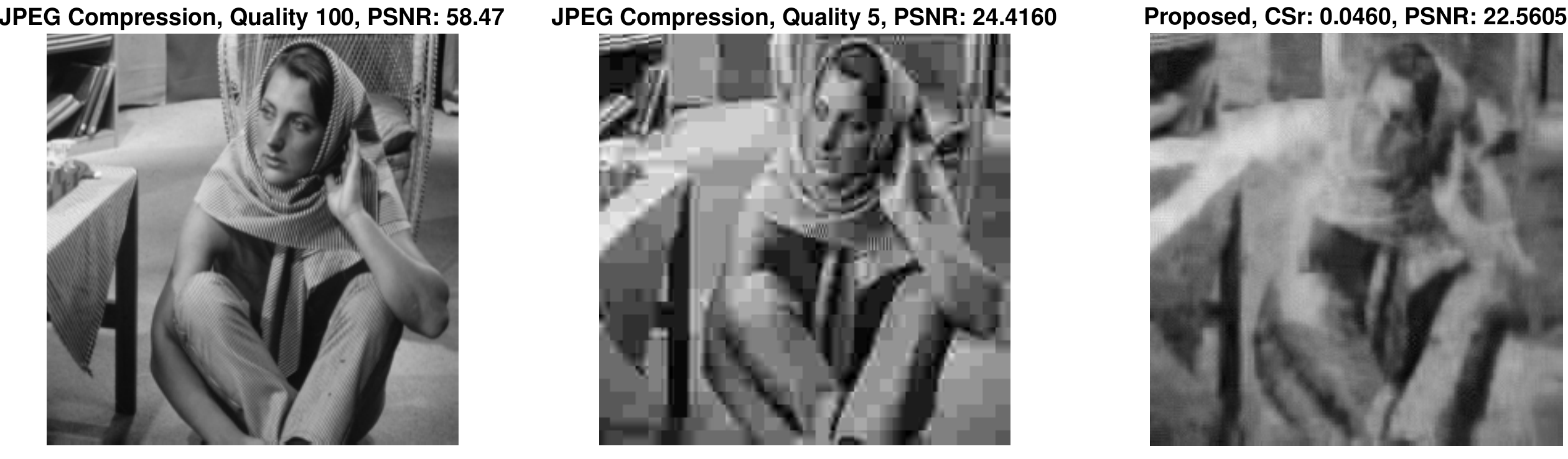}\\
	~\includegraphics[width=0.8\textwidth]{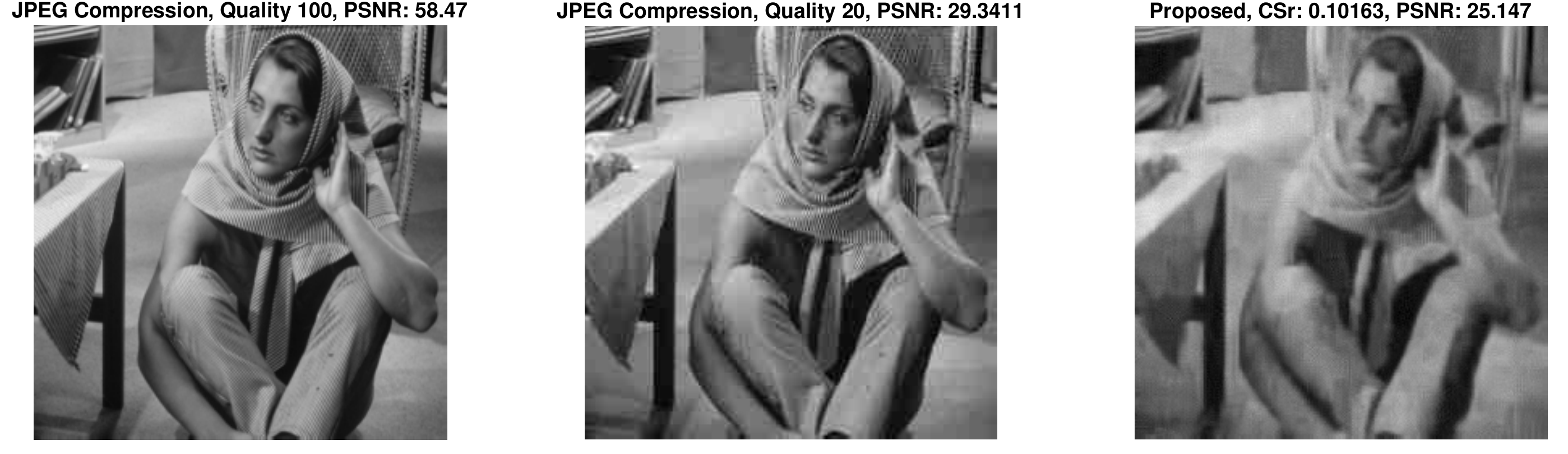}\\
	\vspace{-3mm}
	\caption{Example images: JPEG compared with the proposed SLOPE algorithm. }
	\label{fig:JPEG_comp}
\end{figure}
\begin{table}[tbp!]
	\caption{JPEG compression at different qualities compared with the proposed compressive sensing recovery}
	\centering
	{\scriptsize
		\begin{tabular}{|c|c|c||c|c||c|}
			\hline \multicolumn{3}{|c||}{JPEG Compression} & \multicolumn{2}{|c||}{SLOPE Reconstruction} & Difference  \\
			\hline  Quality & Size (bytes) & PSNR (dB) 
			&  CSr  &  PSNR (dB) & PSNR (dB)\\
			\hline \hline
			1 & 1,563 & 22.0683 & 0.0334 & 21.6318 & 0.4365\\
			\hline
			3 & 1,716 & 22.7278 & 0.0367 & 21.8601 & 0.8677\\
			\hline
			5 & 2,153 & 24.4160 & 0.0460 & 22.5605 & 1.8555\\
			\hline
			7 & 2,559 & 26.0974 & 0.0547 & 23.1834 & 2.9140 \\
			\hline
			9 & 2,946 & 26.9623 & 0.0630 & 23.5511 & 3.4112\\
			\hline
			10 & 3,141 & 27.2853 & 0.0672 & 23.8445 & 3.4408\\
			\hline
			12 & 3,494 & 27.8379 & 0.0747 & 24.1288 & 3.7091 \\
			\hline
			14 & 3,815 & 28.2981 & 0.0816 & 24.4446  & 3.8535 \\
			\hline
			16 & 4,160 & 28.6912 & 0.0890 & 24.7085  & 3.9827\\
			\hline
			18 & 4,478 & 29.0306 & 0.0958 & 24.9645 & 4.0661\\
			\hline
			20 & 4,752 & 29.3411 & 0.1016 & 25.1470 & 4.1941\\
			\hline	
		\end{tabular}}
		\label{Table:JPEG_PSNR}
	\end{table}	
\subsection{Compare with JPEG Compression}
	We now compare SLOPE under the compressive sensing framework with the JPEG compression, which is based on the sparsity of the DCT coefficients in $8\times 8$ blocks.
	We first use a PNG file as the truth and then use the script within MATLAB ``imwrite($\cdot$)" by choosing 8-bits `jpeg' compression with different qualities (100 denotes the highest quality).
	We treat the quality 100 as the standard full file size. For the `Barbara' image we used here, PSNR = 58.47dB (w.r.t. the PNG file) and the file size is 45.6KB at quality 100.
	The compressed image is obtained by changing the compression quality from 1 to 100 and we compare the file size with the full size at quality 100, computing the CSr used in this paper.
	
	Table~\ref{Table:JPEG_PSNR} summarizes the results of JPEG compression compared with the results obtained by our algorithm.
	This is a rough, high level comparison because JPEG also performs an entropy encoding after the DCT transform and quantization, while in our method, the number of compressive measurements is compared with the number of total pixels, and we did not consider the entropy coding on quantized measurements. 
	When the compression is high (lower CSr), the gap between our approach and the JPEG compression is very small ($<0.5$dB). 
	When the compression gets lower, the gap becomes larger.
	One possible reason is that when JPEG is performed on the image, the truth is available and it is very easy to capture useful information from the truth.
	However, under the compressive sensing framework and using the current algorithm, increasing a few number of measurements can help the reconstruction, but not that significantly. 
	Example images can be found in Figure~\ref{fig:JPEG_comp}. It can be seen that JPEG compression has obvious block artifacts while the results of the proposed algorithm become better progressively with increasing number of measurements.
	
	Furthermore, in JPEG compression, if we lose some bits, we may not be able to decode entire blocks. By contrast, in our compressive sensing framework, if we lose some measurements, we can still reconstruct the image, maybe not at a high fidelity.

	\begin{figure}[htbp!]
		\centering
		\includegraphics[width=0.4\textwidth, height = 4cm]{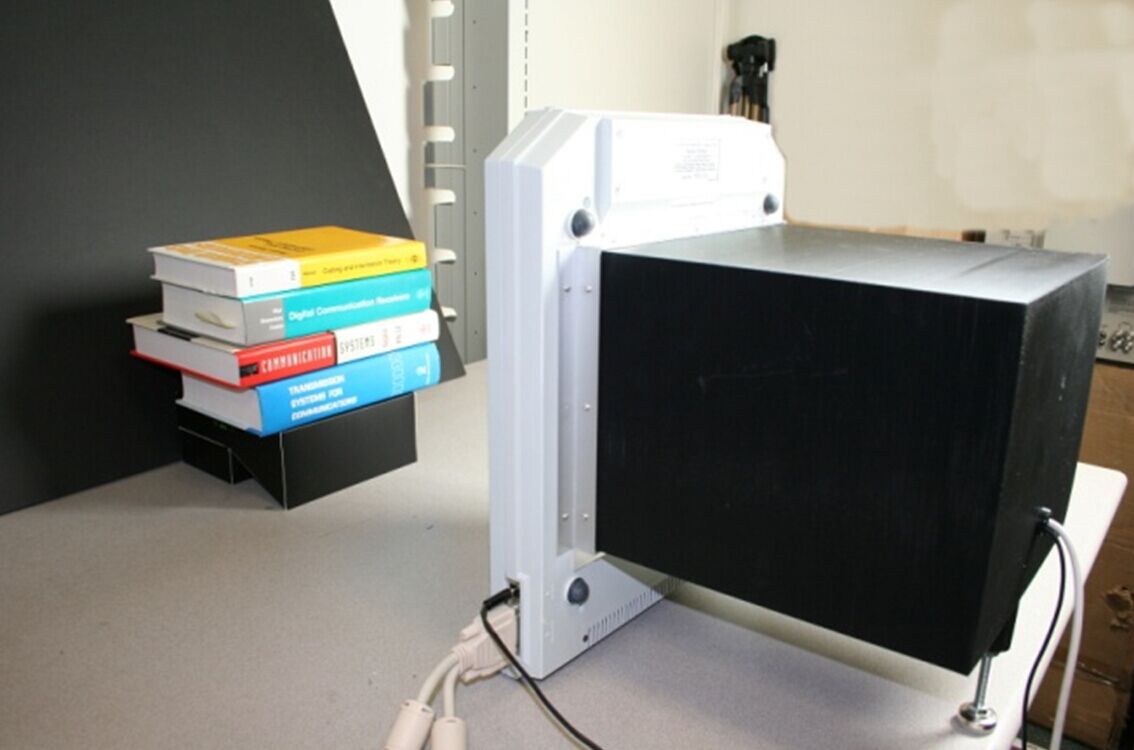}~~
		\includegraphics[width=0.4\textwidth, height =4cm]{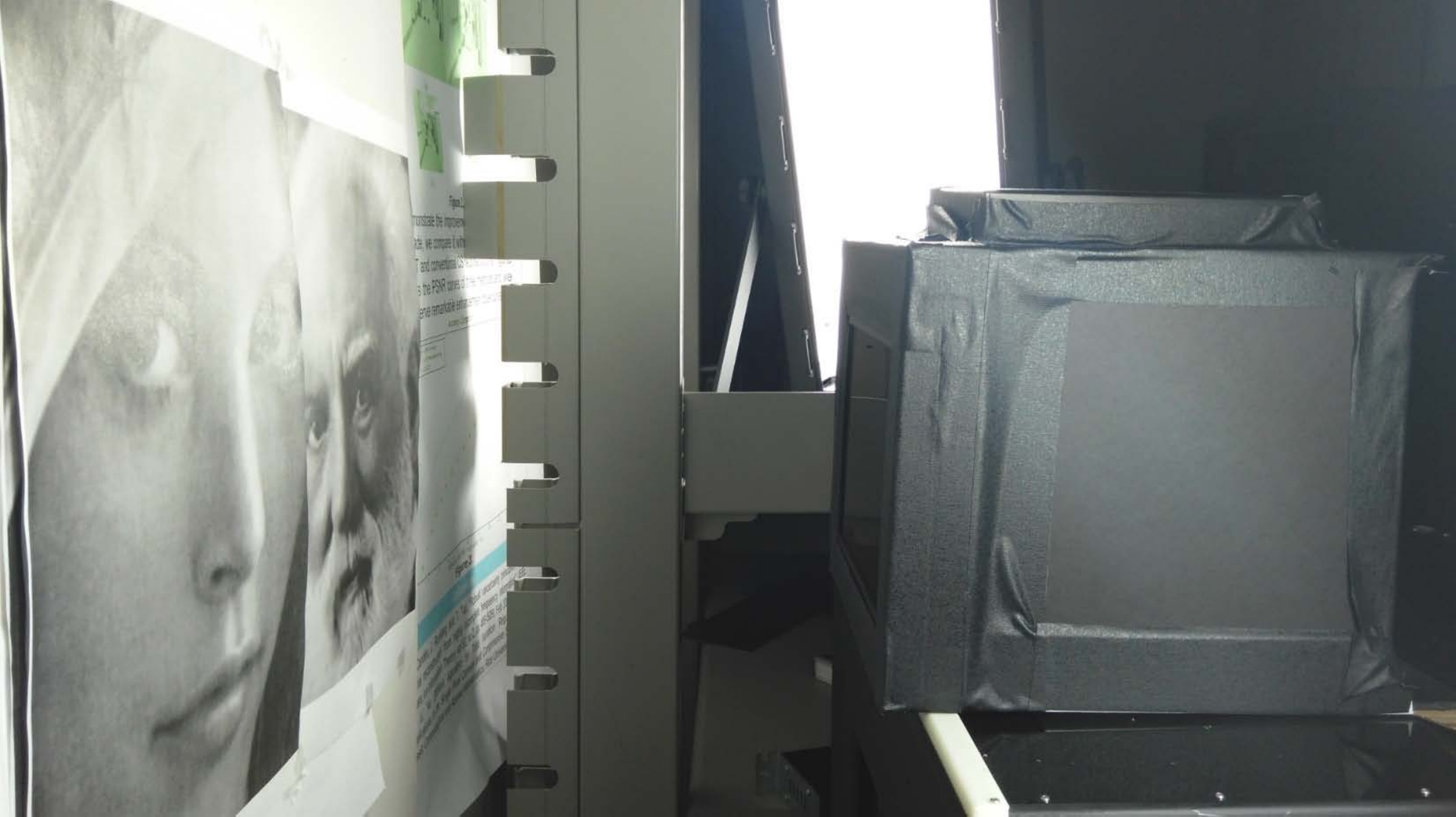}
		\caption{Prototype device. Left: the first generation in~\cite{Huang13ICIP}; right: the current version.}
		\label{fig:camera}
	\end{figure}

\section{Experimental Hardware and Real Data Results}
	\label{Sec:Hardware}
	We have implemented our camera (Figure~\ref{fig:camera}) by using a transparent LCD monitor (with HD resolution)\footnote{http://crystal-display.com/products/transparent-lcd/} as the aperture assembly, which can be programmed by using a programing code implemented on a computer.
	One advantage of using the LCD monitor is that we can control the image resolution by merging the neighbor pixels into the same aperture element. 
	We implemented our control code in C++ to program the LCD. The sensor used in our work is the same as in~\cite{Huang13ICIP}\footnote{http://www.digikey.com/catalog/en/partgroup/tsl2571-evaluation-module/33050}.
	
	There are some  practical issues when we process the real data, especially on the calibration of the sensing matrix $\Amat$. Specifically, for the binary Hadamard matrix employed in our camera, we use the $\{0,1\}$ entries, denoted as $\Amat^+$ (as the ideal case), which differs the $\{-1,1\}$ (normalized to a constant) as generally used Hadamard matrix (since $-1$ can not be implemented in the hardware in a single step). Let $\Amat^0$ denote the Hadamard matrix consisted of $\{-1,1\}$:
	\begin{eqnarray}
	\Amat^+ = \frac{\Amat^0 + {\bf 1}}{2}, \label{eq:A_+}
	\end{eqnarray}
	where ${\bf 1}$ denotes the all-one matrix (every entry is 1) with the same size of $\Amat^0$.
	In the real system, however, even we assume that each element of the aperture assemble is uniform (the transmission rate is the same), there is still some light transmitting the aperture elements where we set to zero. Therefore, for the $(i,j)$-th element of the real sensing  matrix $\Amat$
	\begin{eqnarray}\label{eq:A_real}
	A_{i,j} = \left\{\begin{array}{lll}
	g A^+_{i,j} & {\text {if}} & A^+_{i,j}=1,\\
	f & {\text {if}} & A^+_{i,j}=0,
	\end{array}\right.
	\end{eqnarray}
	where $g$ is a constant to normalize the intensity of the light when the aperture element is programmed to one and $f$ is a constant denoting the intensity of the light when the aperture element is programmed to zero.
	Fortunately, the fast Hadamard transformation can also be used by integrating (\ref{eq:A_+}) and (\ref{eq:A_real}). Moreover, $g$ and $f$ in (\ref{eq:A_real}) can be normalized to a constant.
	In all the real data, we set the step size $\xi = 1.5$.
	These results demonstrate the robustness of the SLOPE algorithm in the noisy case, since noise always exists in the real measurements.

	\begin{figure*}[htbp!]
		\centering
		\includegraphics[width=\textwidth]{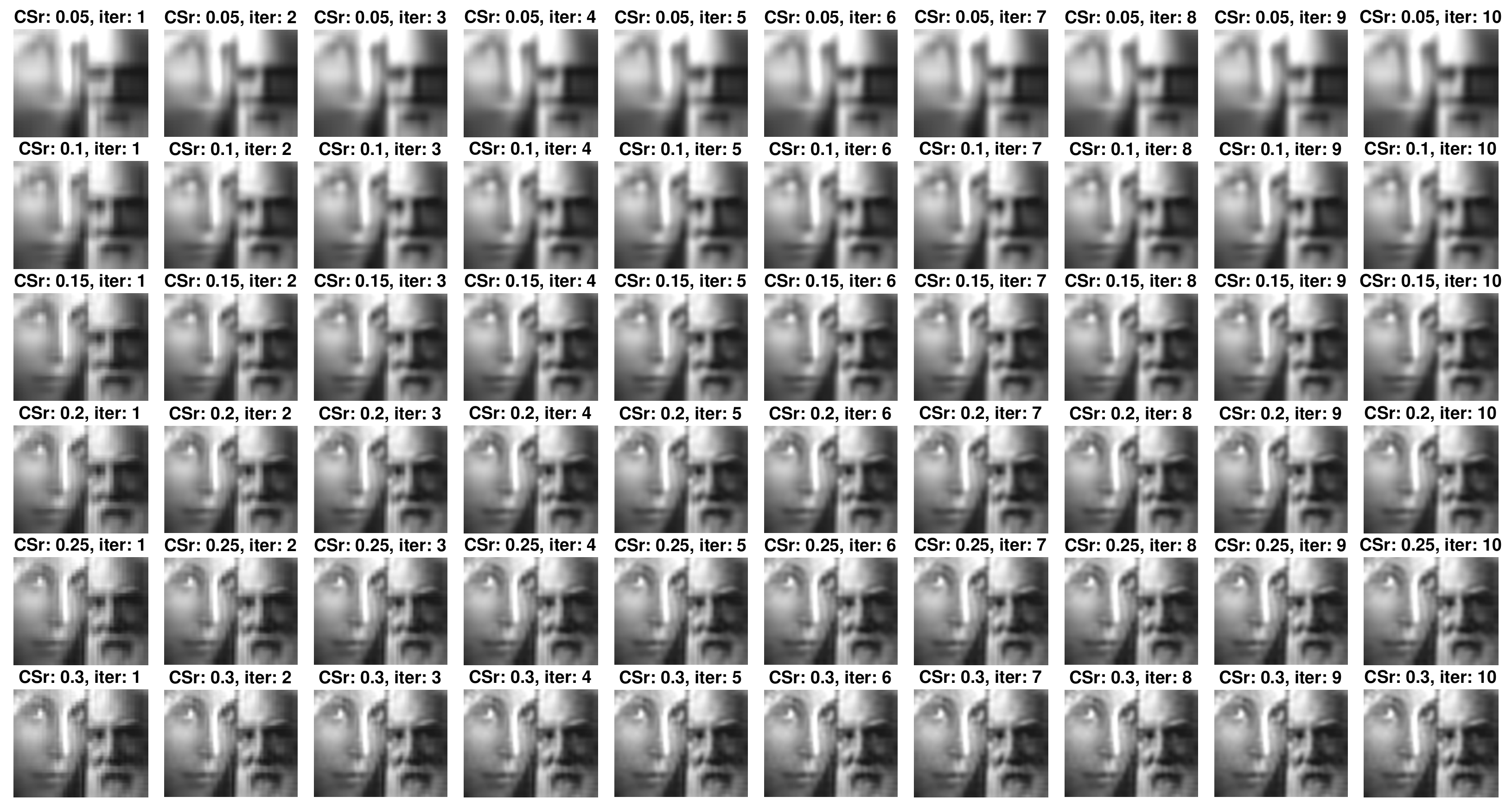}
		\caption{Real data: reconstruction results of different CSr (each row) at each iteration (each column). The image is of size $64\times 64$. }
		\label{fig:4k_real}
	\end{figure*}	

\subsection{Anytime Verification by Real Data}
\label{Sec:anytime_real}
We first consider the gray scale sensor and the image resolution of $64\times 64$ to verify the anytime property of SLOPE.
To capture compressive measurements, we use a sensing matrix which is constructed from rows of a Hadamard matrix of order $N=2^{12}$. Each row of the Hadamard matrix is permuted according to a predetermined random permutation.
The scene is composed of a photo (of two persons) printed on a paper and we capture the measurements of this photo.
Reconstruction results at various CSr using SLOPE are shown in Figure~\ref{fig:4k_real}.
For each row, we plot results at different iterations. It can be seen that the results are getting better with increasing iterations (from left to right); thus, {\em anytime}.
We further notice that only using 2 or 3 iterations (the second and third column), SLOPE can provide descent results. This verifies the performance of the proposed algorithm.
Moreover, our camera along with SLOPE can present very good results at CSr = 0.15.
This further verifies the performance of our hardware system as well as the algorithm.

	\begin{figure*}[htbp!]
		\centering
		\includegraphics[width=\textwidth]{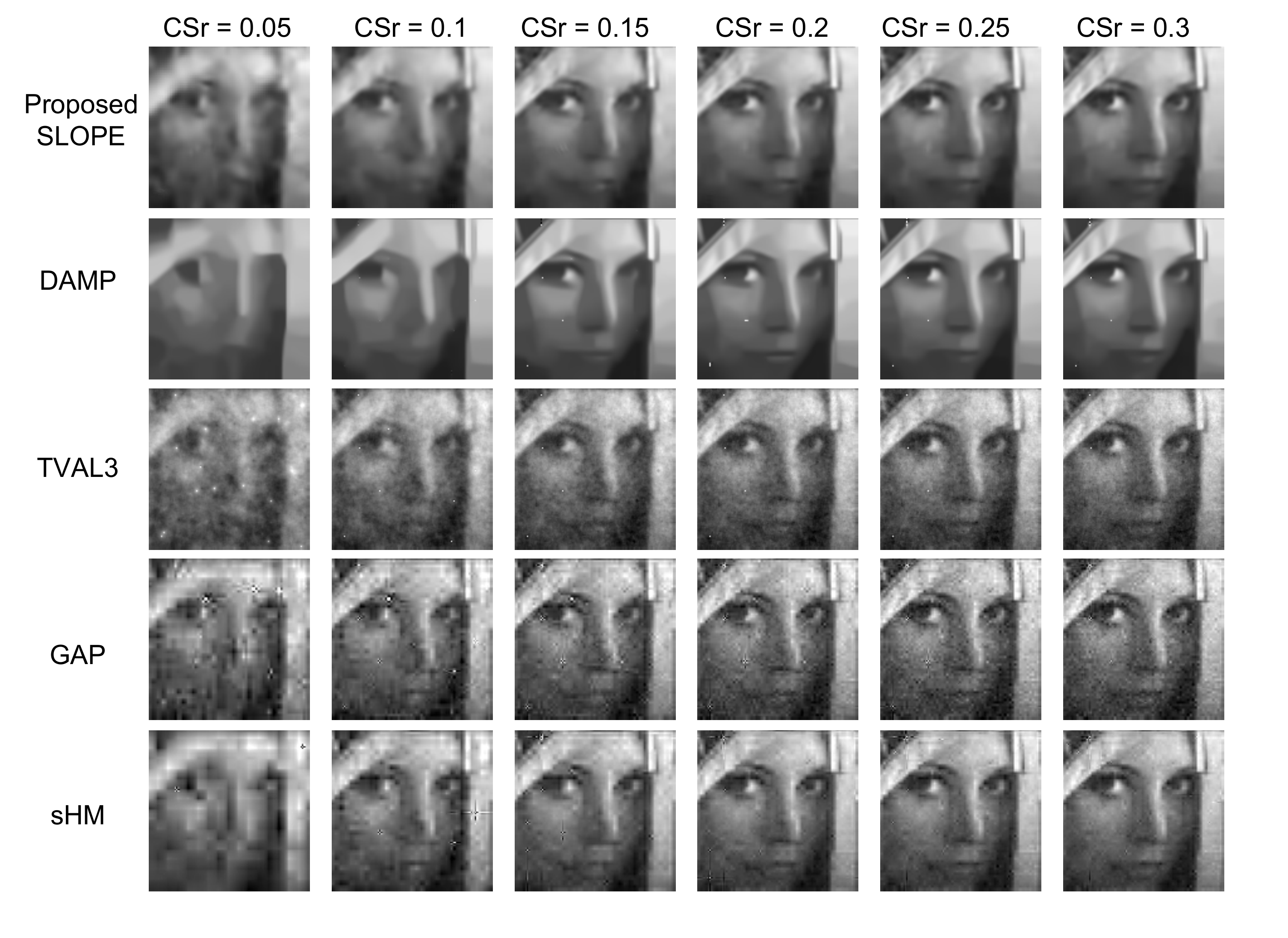}
		\caption{Real data: reconstruction results at different CSr with various algorithms. The image is of size $128\times 128$. }
		\label{fig:Lena_real}
	\end{figure*}

\subsection{Compare with Other Algorithms}
	Next, we consider the case with gray scale sensor and the image resolution of $128\times 128$. 
	To capture compressive measurements, we use a sensing matrix which is constructed from rows of a Hadamard matrix of order $N=2^{14}$. Similar to Section~\ref{Sec:anytime_real}, each row of the Hadamard matrix is permuted according to a predetermined random permutation.
	The scene is composed of a photo printed on a paper and we capture the measurements of this photo.
	Example results using different numbers of measurement are shown in Figure~\ref{fig:Lena_real}. 
	We compare the five algorithms used in the simulation.
	It can be seen that, similar to simulation results, SLOPE provides best results compared to other algorithms when CSr is small. Especially, at CSr = 0.05 and 0.1, SLOPE preserves many details of the face, for example, the left eye of ``Lena".
	DAMP introduces some ``blob" noise because the BM3D denoising approach is used.
	Surprisingly, sHM now works better than TVAL3 and GAP. This is due to the following two reasons. Firstly, the tree structure in wavelet helps the reconstruction and secondly, the Bayesian framework developed in~\cite{Yuan14TSP} is very robust to noise; it infers noise from the measurements. 
	We further observed that the algorithm developed in~\cite{Yuan14TSP} is very helpful to remove spiky noise during reconstruction.
	
		\begin{figure*}[htbp!]
			\centering
			\includegraphics[width=\textwidth]{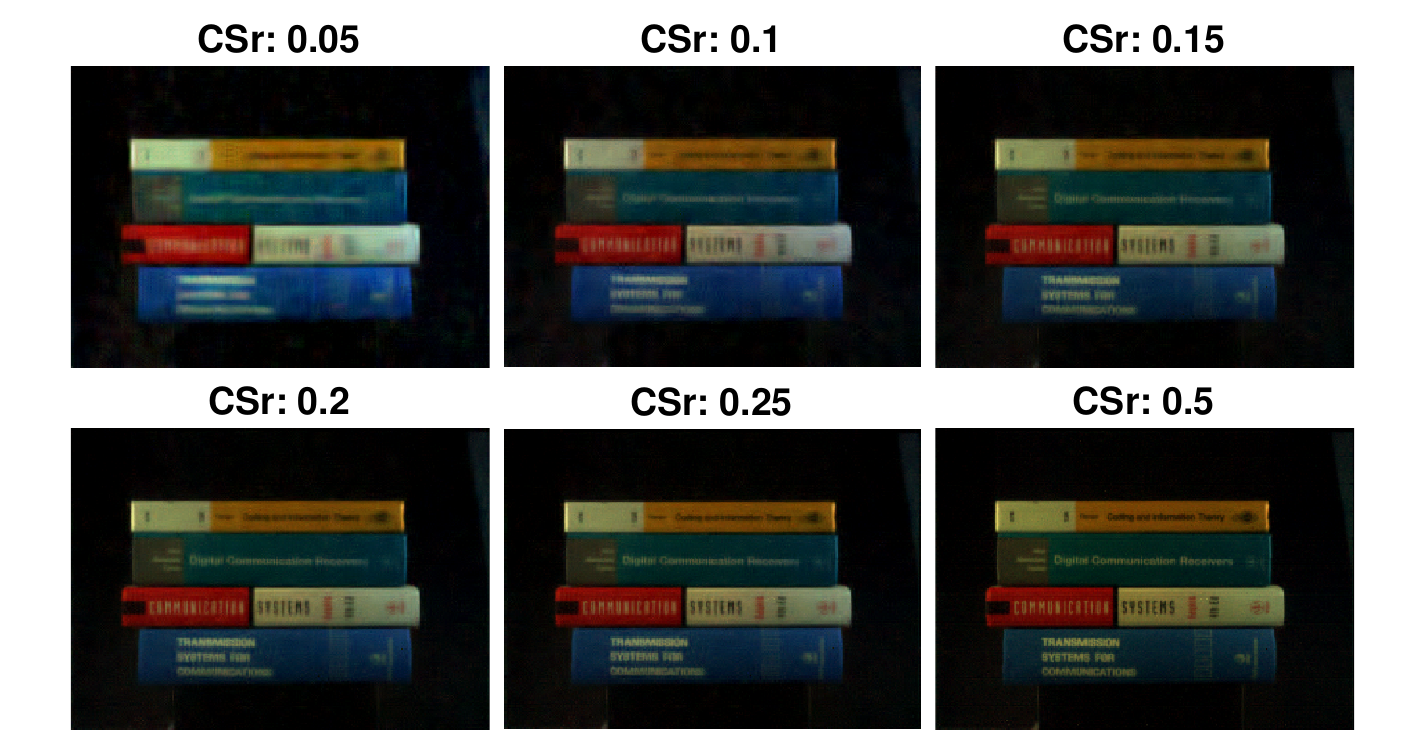}
			\caption{Real data: reconstruction results at different CSr with the proposed SLOPE algorithm. The image size is $217\times 302 \times 3$.}
			\label{fig:Books_real}
		\end{figure*}
\subsection{RGB Images}	
	Next we consider the RGB images captured by a tricolor sensor, with now a resolution of $217\times 302 \times 3$. The sensing matrix is constructed from rows of a Hadamard matrix of order $N=2^{16}$ and the first $65534$ elements are used.
	The scene is the real scene of four books as shown in Figure~\ref{fig:camera}.
	The reconstruction result is shown in Figure~\ref{fig:Books_real} with various compressive sensing ratio.
	
	Note that by using compressive sensing, we can save the sensors as well as the bandwidth. As stated before, we may progressively get better results by receiving more measurements.
	One application of compressive imaging is to get features in limited data by using a small bandwidth. 
	From the results in Figure~\ref{fig:Lena_real}, we may identity high quality features from the reconstructed image at CSr around 0.1.
	If we want to get some details, for example, the book titles in Figure~\ref{fig:Books_real}, we may need CSr around 0.2.  
	On the other hand, if we only need to identify that these are ``books" in Figure~\ref{fig:Books_real}, CSr at 0.05 may be sufficient.

\section{Conclusion}
	\label{Sec:Col}
	An architecture for lensless compressive imaging is proposed. The architecture allows flexible implementations to build simple, reliable imaging devices with reduced size, cost and complexity. Furthermore, the images from the architecture do not suffer from such artifacts as blurring due to defocus of a lens.
	A prototype camera was built using low cost, commercially available components to demonstrate that the proposed architecture is indeed feasible and practical. 
	A new compressive sensing reconstruction algorithm is proposed to achieve excellent results on both simulation and real data. The proposed algorithm enjoys the anytime property and thus provides better results as the computation increases.
	The algorithm is further compared with the JPEG compression and it demonstrates good results at high compression rates.
	Extensive results verified the performance of the algorithm as well as the imaging system. 

\appendix{Proof of Theorem~\ref{thm:anytime}}	
For simplification, we consider the case without clustering patches, and it is readily to extend to the clustering patches case.
We define the following operations for convenience:
\begin{definition}
	\begin{eqnarray}
	\tilde{\xv} &=& {\cal H}\betav \label{eq:Tinv},\\
	\alphav &=& {\cal H}^{-1} {\xv}\label{eq:T},
	\end{eqnarray}	
	where ${\cal H}^{-1}$ denotes the 2D or 3D transformation performed on the image patches including the pixel average matrix. 
\end{definition}

We further define:
\begin{definition}
	\begin{equation}
	\betav_k = \alphav_k + \zetav_k, \label{eq:deltav}
	\end{equation}
	since $\betav_k$ is obtained from the shrinkage of $\alphav_k$.
\end{definition}

Since we assume $\alphav^*$ is sparse, we define the following set:
\begin{definition}
	\begin{equation}
	{\cal I}_+~:~ \forall i,  {\text{ that }} |\alpha^*_i|>0, \label{eq:I+}\\
	\end{equation}
	where $\alpha^*_i$ denotes the $i$-th entry of $\alphav$.
\end{definition}

We further define the sets:
\begin{definition}
	\begin{eqnarray}
	{\cal J}_+&: &\forall i,  {\text{ that }} \lambda_k < |\alpha_{k,i}|, \label{eq:J+}\\
	{\cal J}_-&: &\forall i,  {\text{ that }} \lambda_k \ge |\alpha_{k,i}|, \label{eq:J-}
	\end{eqnarray}
	where $\alpha_{k,i}$ denotes the $i$-th entry of $\alphav_k$.
\end{definition}

\begin{proof}
	We start our derivation from \eqref{eq:slope_x}
	\begin{eqnarray}
	\xv_{k+1} &=& \tilde{\xv}_k + \xi \Amat\ts (\Amat \Amat\ts)^{-1} (\yv - \Amat \tilde{\xv}_{k}),
	\end{eqnarray}
	and we first prove that, $\{\|\xv_{k}-\xv^*\|_2^2\}_{k=1}^\infty$ is a monotonically non-increasing sequence.
	
	Recall that $\yv = \Amat \xv^*$,  we have
	\begin{eqnarray}
	\xv_{k+1}-\xv^* &=& \tilde{\xv}_k - \xv^* + \xi \Amat\ts (\Amat \Amat\ts)^{-1} (\yv - \Amat \tilde{\xv}_{k})\label{eq:txv}
	\end{eqnarray}
	It is equivalent to prove $\{\|\alphav_{k}-\alphav^*\|_2^2\}_{k=1}^\infty$ is a monotonically non-increasing sequence.
	
	Following (\ref{eq:Tinv})-(\ref{eq:T}),
	\begin{eqnarray}
	{\cal H}^{-1}(\xv_{k+1}-\xv^*) 
	&=&{\cal H}^{-1}(\tilde{\xv}_k - \xv^*) + \xi {\cal H}^{-1} \left[\Amat\ts(\Amat \Amat\ts)^{-1} \Amat (\xv^* -  \tilde{\xv}_{k})\right],
	\end{eqnarray} 
	which is equivalent to (please refer to (\ref{eq:thetav_k+1})-(\ref{eq:xk+1_Tw}))
	\begin{eqnarray}
	\alphav_{k+1} - \alphav^* = \alphav_{k} - \alphav^* + \zetav_k + \xi \Rmat\ts(\Rmat \Rmat\ts)^{-1}\Rmat(\alphav^* - \alphav_{k})
	\end{eqnarray}
	
	Therefore,
	\begin{eqnarray}
	\|\alphav_{k+1}-\alphav^*\|_2^2 
	&=&\|\alphav_k - \alphav^*\|^2_2  + 2(\alphav_k - \alphav^*)\ts[{\zetav}_k  + \xi \Rmat\ts(\Rmat \Rmat\ts)^{-1}\Rmat (\alphav^* -  \betav_{k}) ]\nonumber\\
	&&+ \|{\zetav}_k 
	+ \xi \Rmat\ts(\Rmat \Rmat\ts)^{-1}\Rmat (\alphav^* -  \betav_{k})\|_2^2. \label{eq:xkk+1}
	\end{eqnarray}
	What we want to prove is that
	\begin{eqnarray}
	\|\alphav_{k+1}-\alphav^*\|_2^2 
	-\|\alphav_k - \alphav^*\|^2_2 &\stackrel{\rm def}{=}& r_k \le 0 .
	\end{eqnarray}

	From (\ref{eq:xkk+1}),
	\begin{align}
	r_k &= 2(\alphav_k - \alphav^*)\ts[{\zetav}_k  + \xi \Rmat\ts(\Rmat \Rmat\ts)^{-1} (\alphav^* -  \betav_{k}) ]+ \|{\zetav}_k  + \xi \Rmat\ts(\Rmat \Rmat\ts)^{-1} (\alphav^* -  \betav_{k})\|_2^2 \nonumber\\
	&= 2(\betav_{k} - \alphav^* - {\zetav}_{k})\ts[{\zetav}_k  + \xi \Rmat\ts (\Rmat \Rmat\ts)^{-1}\Rmat (\alphav^* -  \betav_{k}) ]\|{\zetav}_k  + \xi \Rmat\ts(\Rmat \Rmat\ts)^{-1}\Rmat (\alphav^* -  \betav_{k})\|_2^2  \nonumber \\
	&= 2(\betav_{k} - \alphav^* )\ts[{\zetav}_k  + \xi \Rmat\ts (\Rmat \Rmat\ts)^{-1} \Rmat(\alphav^* -  \betav_{k})] - 2{\zetav}_{k}\ts[{\zetav}_k  + \xi \Rmat\ts (\Rmat \Rmat\ts)^{-1} \Rmat(\alphav^* -  \betav_{k})]\nonumber\\
	&~~+ [{\zetav}_k  + \xi \Rmat\ts (\Rmat \Rmat\ts)^{-1}\Rmat (\alphav^* -  \betav_{k})]\ts[{\zetav}_k  + \xi \Rmat\ts (\Rmat \Rmat\ts)^{-1} (\alphav^* -  \betav_{k})]\nonumber\\
	&= 2(\betav_{k} - \alphav^* )\ts{\zetav}_k  + 2\xi (\betav_{k} - \alphav^* )\ts\Rmat\ts (\Rmat \Rmat\ts)^{-1} (\alphav^* -  \betav_{k})  - \|{\zetav}_{k}\|^2_2 + \|\xi \Rmat\ts (\Rmat \Rmat\ts)^{-1} \Rmat(\alphav^* -  \betav_{k})\|_2^2\nonumber\\
	&= 2(\betav_{k} - \alphav^* )\ts{\zetav}_k + [ - \|{\zetav}_{k}\|^2_2]+ [(\xi-1)^2-1] \|\Rmat\ts (\Rmat \Rmat\ts)^{-1} \Rmat(\alphav^* -  \betav_{k}) \|^2_2. \label{eq:rk}
	\end{align} 

	There are three terms on the right-hand side of (\ref{eq:rk}):
	\begin{itemize}
		\item The second term $[ - \|{\zetav}_{k}\|^2_2]$ is obviously non-positive.
		\item The third term $[(\xi-1)^2-1] \|\Rmat\ts (\Rmat \Rmat\ts)^{-1} \Rmat(\alphav^* -  \betav_{k}) \|^2_2$ is non-positive when $\xi\in (0,2]$, given $\xi >0$.
	\end{itemize}
	In the following, we only need to prove the first term is non-positive.
	%
	%
	%
	\begin{align}
	& 2(\betav_{k} - \alphav^* )\ts{\zetav}_k =   
	2\left(\alphav_k\odot\max\left\{1-\frac{\lambda_k}{|\alphav_k|},0\right\} - \alphav^*\right)\ts  \left(-\alphav_k\odot\min\left\{1,\frac{\lambda_k}{|\alphav_k|}\right\}\right) \nonumber\\
	&= 2\left(\alphav_k\odot\max\left\{1-\frac{\lambda_k}{|\alphav_k|},0\right\} \right)\ts \left(-\alphav_k\odot\min\left\{1,\frac{\lambda_k}{|\alphav_k|}\right\}\right) -2 (\alphav^*)\ts \left(-\alphav_k\odot\min\left\{1,\frac{\lambda_k}{|\alphav_k|}\right\}\right). \label{eq:2terms}
	\end{align}
	where $|\cdot|$ is the absolute value performed on each element and $\odot$ is the element-wise product.

	Therefore, for the values that are not in ${\cal I}_+$, the second term of (\ref{eq:2terms}) will be zero.
	For the values in ${\cal J}_-$, the first term of (\ref{eq:2terms}) will be zero. 
	
	With the definitions of ${\cal J}_+, {\cal J}_-$ and ${\cal I}_+$:
	\begin{align}
	\nonumber
	& 2(\betav_{k} - \alphav^* )\ts{\zetav}_k= 2\sum_{i\in {\cal J}_+} (\lambda_k - |\alpha_{k,i}|)\lambda_k + 2\lambda_k\sum_{i\in {\cal J}_+} \alpha^*_i{\rm sign}(\alpha_{k,i}) + 2\sum_{i\in {\cal J}_-} \alpha^*_i\alpha_{k,i} \nonumber\\
	&= 2\lambda_k\left\{\sum_{i\in {\cal J}_+} (\lambda_k - |\alpha_{k,i}|) + \sum_{i\in {\cal J}_+} \alpha^*_i{\rm sign}(\alpha_{k,i}) + \sum_{i\in {\cal J}_-} \alpha^*_i\frac{\alpha_{k,i}}{\lambda_k}\right\}.
	\end{align}
	Given $\lambda_k\ge0$, if we want to prove $2(\betav_{k} - \alphav^* )\ts{\zetav}_k\le0$, we need:
	\begin{eqnarray}
	\sum_{i\in {\cal J}_+} (\lambda_k - |\alpha_{k,i}|) + \sum_{i\in {\cal J}_+} \alpha^*_i{\rm sign}(\alpha_{k,i}) + \sum_{i\in {\cal J}_-} \alpha^*_i\frac{\alpha_{k,i}}{\lambda_k} \le 0,
	\end{eqnarray}
	which is equivalent to
	\begin{eqnarray}
	\sum_{i\in {\cal J}_+} \alpha^*_i{\rm sign}(\alpha_{k,i}) + \sum_{i\in {\cal J}_-} \alpha^*_i\frac{\alpha_{k,i}}{\lambda_k} \le \sum_{i\in {\cal J}_+} (|\alpha_{k,i}| - \lambda_k). \label{eq:alpha}
	\end{eqnarray}
	\begin{lemma}
		$\|\alphav^*\|_1$ is an upper bound of $\sum_{i\in {\cal J}_+} \alpha^*_i{\rm sign}(\alpha_{k,i}) + \sum_{i\in {\cal J}_-} \alpha^*_i\frac{\alpha_{k,i}}{\lambda_k}$.
	\end{lemma}
	\begin{proof}
		\begin{align}
		\sum_{i\in {\cal J}_+} \alpha^*_i{\rm sign}(\alpha_{k,i}) &= \sum_{i\in {\cal J}_+\cap {\cal I}_+}\alpha^*_i{\rm sign}(\alpha_{k,i}) \le \sum_{i\in {\cal J}_+\cap {\cal I}_+}|\alpha^*_i| \\
		\sum_{i\in {\cal J}_-} \alpha^*_i\frac{\alpha_{k,i}}{\lambda_k}& \le \sum_{i\in {\cal J}_-} |\alpha^*_i|\frac{|\alpha_{k,i}|}{\lambda_k} \le  \sum_{i\in {\cal J}_-\cap {\cal I}_+}|\alpha^*_i| 
		\end{align}
		Therefore
		\begin{align}
		\sum_{i\in {\cal J}_+} \alpha^*_i{\rm sign}(\alpha_{k,i}) + \sum_{i\in {\cal J}_-} \alpha^*_i\frac{\alpha_{k,i}}{\lambda_k}\le \|\alphav^*\|_1.
		\end{align}
	\end{proof}
	
	Along with (\ref{eq:alpha}),
	if we select $\lambda_k$ in each iteration such that  $\|\betav_k\|_1\ge \|\alphav^*\|_1$, the first term of (\ref{eq:rk}) will be non-positive.
	Along with the other two non-positive terms, we have proved that
	$r_k\le 0$.
	Till now, we have proved that $\{\|\alphav_{k}-\alphav^*\|_2^2\}_{k=1}^\infty$ is a monotonically non-increasing sequence if (\ref{eq:ell1_ball}) is satisfied and $\xi \in(0,2]$. 
	
	In the following, we prove that $\{\|\alphav_{k}-\alphav^*\|_2^2\}$ converges to zero. 
	We first prove that  $\{\|\Rmat\alphav_{k}-\Rmat\alphav^*\|_2^2\}$ converges to zero, and then with the RIP (restricted isometry property) condition~\cite{cs_Candes06,Candes05,Candes06ITT} on $\Rmat$, $\alphav_{k}$ converges to the true signal $\alphav^*$.
	
	From (\ref{eq:rk}), when $\xi \in(0,2]$,
	\begin{eqnarray}
	r_k &=& 2(\betav_{k} - \alphav^* )\ts{\zetav}_k + [ - \|{\zetav}_{k}\|^2_2]+ [(\xi-1)^2-1] \|\Rmat\ts (\Rmat \Rmat\ts)^{-1} \Rmat(\alphav^* -  \betav_{k}) \|^2_2 \le 0, \label{eq:rk_new}
	\end{eqnarray}
	if the constrain in (\ref{eq:ell1_ball}) is satisfied.
	Now we tight the constrain of $\xi$ to $\xi\in(0,2)$,
	and the third term of $r_k$:
	\begin{eqnarray}
	[(\xi-1)^2-1] \|\Rmat\ts (\Rmat \Rmat\ts)^{-1} \Rmat(\alphav^* -  \betav_{k}) \|^2_2 &<& 0  ~~(r_k <0) \label{eq:r<0}\\
	\quad {\rm or} \quad \|\Rmat\ts (\Rmat \Rmat\ts)^{-1} \Rmat(\alphav^* -  \betav_{k})  \|^2_2 &=& 0.\label{eq:r=0}
	\end{eqnarray}
	Separately consider the following two cases:
	\begin{itemize}
		\item[1)] $\|\Rmat\ts (\Rmat \Rmat\ts)^{-1} \Rmat(\alphav^* -  \betav_{k})  \|^2_2 =0$, implies $\Rmat\ts (\Rmat \Rmat\ts)^{-1} \Rmat(\alphav^* -  \betav_{k}) = 0$ and left-multiplying by $\Rmat$ shows
		\begin{eqnarray} \label{eq:conv_txv_k}
		\|\Rmat(\alphav^* -  \betav_{k})\|_2^2 = \|\yv - \Rmat\betav_{k}\|_2^2 = 0.
		\end{eqnarray}
		Since we impose that $\betav_k$ is sparse, (\ref{eq:conv_txv_k}) means that $\betav_k$ converges to a sparse solution of $\yv = \Amat{\cal H} \alphav^* = \Amat\xv$.
		
		Recall the RIP,
		\begin{eqnarray}
		(1-\delta_s) \|\alphav\|_2^2 \le \|\Rmat_s \alphav\|_2^2 \le (1+\delta_s)\|\alphav\|_2^2,
		\end{eqnarray}
		where $\delta_s \in(0,1)$ is a constant and $\Rmat_s$ is a subset of $\Rmat$.
		$(\alphav_k -\alphav^*)$ has at most $(m^*_{\lambda_k} + K^*)$ non-zero elements (recall the selection of $\lambda_k$ in (\ref{eq:def_m_star})) and we assume $m^*_{\lambda_k} + K^* < N_c$, where $K^*$ is the number of non-zero elements in $\alphav^*$.
		Define 
		\begin{eqnarray}
		\delta &=& \inf \{\delta_s, \forall s = 1, \dots, m^*_{\lambda_k} + K^*\}.
		\end{eqnarray}
		Imposing the RIP on (\ref{eq:conv_txv_k}), we have
		\begin{eqnarray}\label{eq:conv_rip}
		(1-\delta) \|\alphav^* -  \betav_{k}\|_2^2 \le 0.
		\end{eqnarray}
		The RIP condition required for our proof is $0<\delta_{m^*_{\lambda_k} + K^*} <1$. 
		If we select $m^*_{\lambda_k} = K^*$ (which is a lower bound of $m^*_{\lambda_k}$), we have $0<\delta_{2K^*}<1$, which is the same as the condition derived in~\cite{WangAIT15}.
		Under this condition, (\ref{eq:conv_rip}) means
		\begin{eqnarray}
		\|\alphav^* -  \betav_{k}\|_2^2 = 0.
		\end{eqnarray}
		This means $\betav_{k}$ converges to $\alphav^*$ and in this case $\betav_k =\alphav_{k} $.
		This further means $\tilde{\xv}_{k}$ converges to $\xv^*$ and  $\xv_k =\tilde{\xv}_{k} $.
		\item[2)] 
		$\|\Rmat\ts (\Rmat \Rmat\ts)^{-1} \Rmat(\alphav^* -  \betav_{k}) \|^2_2>0$ and therefore $r_k <0$ strictly.
		In this case,
		\begin{eqnarray}
		\|\alphav_{k+1}-\alphav^*\|_2^2 
		-\|\alphav_k - \alphav^*\|^2_2 &=& r_k <0 \nonumber\\
		\|\alphav_{k+1}-\alphav^*\|_2^2  = \|\alphav_k - \alphav^*\|^2_2 + r_k &<& \|\alphav_k - \alphav^*\|^2_2
		\end{eqnarray}
		Since $\|\alphav_{k}-\alphav^*\|_2^2 \ge 0$, and it is now a decreasing sequence, 
		\begin{eqnarray}
		\lim_{k \rightarrow \infty}\|\alphav_{k}-\alphav^*\|_2^2 &=& {\rm Const}~ (\ge 0),
		\end{eqnarray}
		and thus 
		\begin{eqnarray}
		\lim_{k \rightarrow \infty} r_k &=& 0.
		\end{eqnarray}
		Since all the three terms of $r_k$ are non-positive, $r_k \rightarrow 0$ means all the three terms approach zero, specifically
		\begin{eqnarray}
		\lim_{k \rightarrow \infty}{\zetav}_k= 0 ~~{\rm and}~~ \lim_{k \rightarrow \infty}\|\Rmat\ts (\Rmat \Rmat\ts)^{-1} \Rmat(\alphav^* -  \betav_{k}) \|^2_2= 0.
		\end{eqnarray}
		$\lim_{k \rightarrow \infty}\zetav_k= 0$ can be obtained via $\lim_{k \rightarrow \infty}\lambdav_k= 0$ and in this case $\lim_{k \rightarrow \infty}\betav_{k}= \alphav_{k} $.
		
		$ \lim_{k \rightarrow \infty}\|\Rmat\ts (\Rmat \Rmat\ts)^{-1} \Rmat(\alphav^* -  \betav_{k}) \|^2_2= 0$ means $ \lim_{k \rightarrow \infty}\| \Rmat(\alphav^* -  \betav_{k}) \|^2_2= 0$, and 
		similar to (\ref{eq:conv_rip}), when the RIP on $\Rmat$, (\ref{eq:rip_m*}) is satisfied,
		\begin{eqnarray}
		\lim_{k \rightarrow \infty}\|\alphav^* -  \betav_{k}\|_2^2= 0.
		\end{eqnarray}
		This means $\betav_{k}$ converges to $\alphav^*$, which is equivalent to $\xv_k$ converges to $\xv^*$. 
	\end{itemize}
	Therefore, we have proved in both cases ${\xv}_{k}$ ($\alphav_k$) monotonically converges to $\xv^*$ ($\alphav^*$).
	Till now, Theorem~\ref{thm:anytime} has been proved.	
\end{proof}

\bibliographystyle{IEEEtran}

\end{document}